\documentclass{article}

\usepackage{arxiv}
\usepackage[numbers]{natbib} 
\usepackage{enumitem}

\usepackage[utf8]{inputenc} % allow utf-8 input
\usepackage[T1]{fontenc}    % use 8-bit T1 fonts
\usepackage{hyperref}       % hyperlinks
\usepackage{url}            % simple URL typesetting
\usepackage{booktabs}       % professional-quality tables
\usepackage{amsfonts}       % blackboard math symbols
\usepackage{nicefrac}       % compact symbols for 1/2, etc.
\usepackage{microtype}      % microtypography
\usepackage{xcolor}         % colors

\usepackage{microtype}
\usepackage{graphicx}
\usepackage{subfig}

\usepackage{amsmath}
\usepackage{amssymb}
\usepackage{mathtools}
\usepackage{amsthm}

\usepackage[capitalize,noabbrev]{cleveref}
\usepackage{algpseudocode}
\usepackage{algorithm}
\usepackage{wrapfig}
\usepackage{cleveref}
\usepackage{minitoc}
\usepackage{subcaption}

\theoremstyle{plain}
\newtheorem{theorem}{Theorem}[section]
\newtheorem{proposition}[theorem]{Proposition}
\newtheorem{lemma}[theorem]{Lemma}

\theoremstyle{definition}

\theoremstyle{remark}

\DeclareMathOperator*{\argmin}{arg\,min}
\DeclareMathOperator*{\argmax}{arg\,max}

\providecommand{\synres}[2]{\ensuremath{#1 \pm #2}}
\providecommand{\synbest}[2]{\ensuremath{\mathbf{#1 \pm #2}}}

\title{Multi-Marginal Inverse Optimal Transport for Contrastive Learning Via Explicit Anchor-Positive-Negative Coupling}

\author{
 Ngoc-Hai Nguyen \\
  Department of Electrical and Computer Engineering\\
  Tufts University\\
  \texttt{Ngoc-Hai.Nguyen@tufts.edu} \\
   \And
 Thuan Nguyen \\
    Department of Engineering, Engineering Technology \\
    East Tennessee State University \\
    \texttt{nguyent11@etsu.edu} \\
  \And
 Prakash Ishwar \\
    Department of Electrical and Computer Engineering \\
    Boston University\\
    \texttt{pi@bu.edu}
\And
 Shuchin Aeron \\
    Department of Electrical and Computer Engineering \\
    Tufts University\\
    NSF AI Institute for Artificial Intelligence and Fundamental Interactions\\
    \texttt{shuchin@ece.tufts.edu}
 \\
}

\begin{document}
\maketitle
\begin{abstract}
Inverse Optimal Transport (OT) based methods for representation learning learn representations such that the global OT coupling between a pair of data marginals in the representation space, concentrates on the positive pairs. This is in contrast to previous methods that primarily focused on pairwise matching. However, these methods \textit{DO NOT} utilize negative pairs and hence are not truly contrastive in their approach. We show that this leads to issues of dimensional collapse and hence degraded downstream performance. To alleviate this, we develop a novel multi-marginal (MM) inverse OT (IOT) contrastive learning (CL) approach called Neg-MMIOT-CL, which learns representations such that the global multi-marginal OT (MMOT) coupling between a triple of data marginals, with respect to a carefully designed ground-cost between triplets of data points in the representation space, concentrates on the anchor-positive-negative \emph{triplets}. For a latent class model, we empirically show that Neg-MMIOT-CL alleviates dimensional collapse.  Furthermore, for a specific choice of ground cost for all triplets in representation space, we prove that the optimal representation configuration for Neg-MMIOT-CL exhibits equiangular property for within-class and across-class representations, which translates to Neural-Collapse when the representation dimension is larger than the number of classes minus one -- a result that is \emph{previously  established only} for pairwise contrastive learning methods. Finally, we propose Neg-IOT-CL-PushPull, that is a computationally efficient alternative to Neg-MMIOT-CL, alleviating the high cost of computing MMOT plans needed during implementation. We apply these methods on both synthetic and real-world datasets and show significant improvements over existing OT-based contrastive learning methods. 
\end{abstract}

\section{Introduction}
Supervised or Self-supervised contrastive representation learning is based on a simple but effective idea: design a representation map that pulls positive pairs -- (similar examples that \emph{may} have the same latent class) together and pushes negative pairs -- (similar examples that \emph{may} have the different latent classes) apart. There is an inherent tradeoff here that one must solve: align positives while maintaining separation between the negatives. In this context, modern CL methods that use the objectives, such as InfoNCE loss, derived as proxies for maximizing mutual-information between the datum and its representation \cite{tschannen2019mutual}, have proven to be very successful - from SimCLR in vision \cite{chen2020simple} to CLIP in vision-language pretraining \cite{radford2021learning}. 

Recently methods based on Optimal Transport (OT) \cite{peyre2019computational, villani2008optimal} have been used for representation learning  \cite{shi2023understanding, shi2024ot, xue2025protocol, piran2024contrasting}. Instead of the InfoNCE-type losses that measure pairwise alignment, these methods use the cost of OT coupling between the pair of data marginals with respect to a ground-cost derived from similarity in the representation space as a loss, and the learning objective trains the representation map such that the OT coupling concentrates on the given set of positive pairs. Since the ground-cost for the OT depends on the representation map, these methods are collectively referred to as Inverse-OT for CL (IOT-CL). 

Yet, \emph{these OT-based formulations of CL still contain an important gap} in that they emphasize \textbf{\textit{only}} positive correspondence and \emph{do not fully utilize the power of CL that explicitly utilizes negative pairs as well.}

As such, these methods fall into the category of \textbf{\textit{non-contrastive}} methods such as BYOL \cite{grill2020bootstrap} and VICReg \cite{bardes2021vicreg}. The utility of the representation is also determined by such separation is achieved between negative pairs, especially in imbalanced or overlapping data regimes. Recent analyses of CL have shown that this balance between attraction and repulsion is closely tied to the emergence of Neural Collapse and simplex Equiangular Tight Frame (ETF) geometry in the learned embeddings \cite{wang2023towards,nguyen2024neural}. Meanwhile, current IOT-CL methods do not naturally lead to such configurations. Instead, in order to enforce uniformity, negative interactions are induced through normalization or auxiliary regularizers  in \cite{shi2023understanding}, in the absence of which, OT-based objectives lead to feature overlap and degraded downstream discrimination as observed in their results and our Dimensional Collapse Figure \ref{fig:dc_synthetic_affine} in Section \ref{sec: exp}.

This paper closes this gap by explicitly incorporating anchor-positive-negative triplet coupling into IOT-based CL. In particular, we make the following \textbf{main contributions}.

\begin{enumerate}[leftmargin=*]
    \item We propose \emph{Negative Multi-Marginal Inverse Optimal Transport (Neg-MMIOT-CL)} in \Cref{sec:neg-mmot} and its low computational-complexity alternative  
    \emph{Negative Inverse Optimal Transport PushPull (Neg-IOT-CL-PushPull)} in \Cref{sec:neg-iot-pushpull}, 
 that explicitly encode negative relations rather than relying on positive matching alone. This key contribution is conceptual as well as algorithmic: instead of adding an extra uniformity loss to promote separation between negative pairs as done in \cite{shi2023understanding}, we make negative separation part of the IOT-CL formulation itself.
 
    \item Our theoretical analysis (\Cref{sec: analysis}) shows that this modification is not merely heuristic. 
    When ground truth triplets comprise anchor and positive samples in the same latent class and the negative sample in a different class, the latent classes are balanced, the representation features are unconstrained, the cost of coupling a data triplet $(i,j,k)$ is defined by an affine decreasing function of the difference between inner products between representations of $(i,j)$ and $(i,k)$ pairs, and the dimension of the representation space is at least the number of classes minus one, we prove that the global minimizer of Neg-MMIOT-CL exhibits the now hallmark geometry of CL \cite{graf2021dissecting, jiang2024hard, kini2023symmetric}: within-class collapse together with a simplex ETF arrangement of class means. We also prove that for costs that are defined by non-affine decreasing functions, the Neural Collapse configuration is a stationary point for the spherical gradient flow. These results provide a principled bridge between OT and the representation geometry traditionally associated with contrastive objectives \cite{wang2023towards, nguyen2024neural}.
    
    \item In \Cref{sec: exp}, we show that the empirical results are consistent with the theoretical picture. On synthetic Gaussian mixture data, our methods recover Neural Collapse and prevent Dimensional Collapse much more effectively than IOT-CL and non-contrastive baselines such as BYOL \cite{grill2020bootstrap} and VICReg \cite{bardes2021vicreg}. On supervised and unsupervised frameworks in vision and vision-language datasets, both Neg-MMIOT-CL and Neg-IOT-CL-PushPull are consistently competitive and often outperform conventional IOT-CL.
\end{enumerate}
\noindent\textbf{Related work:} While we have adequately covered highly related work in the exposition thus far, a broader related literature survey and further differences from closely related works are detailed in Appendix~\ref{app: related_work} 
\section{Enhancing IOT-CL via negative repulsion mechanism} 
% \sacomment{This seems a little misalinged with version 3 - need to check.}

\noindent \textbf{Notation:} For a set $\mathcal{A}$, let $|\mathcal{A}|$ denote its cardinality and $\mathbf{1}(\mathcal{A})$ its indicator. For $N\in\mathbb{N}$, write $[N]:=\{1,\ldots,N\}$. Let $\Delta([N]^L)$ be the simplex of PMFs on $[N]^L$, and let $\mathbf{u}_N$ denote the uniform PMF on $[N]$. Given marginals $\nu_1,\ldots,\nu_L\in\Delta([N])$, define $\Pi(\nu_1,\ldots,\nu_L)$ as the set of couplings $P\in\Delta([N]^L)$ with these marginals. We write $[N]^L_{\neq}$ for the set of tuples $(i_1,\ldots,i_L)\in[N]^L$ with distinct entries. We view each $P\in\Delta([N]^L)$ as a nonnegative order-$L$ tensor, i.e., $P\in\mathbb{R}_+^{[N]^L}$. For an \textbf{\textit{extended real-valued}} order-$L$ tensor $C$ and $P\in\Delta([N]^L)$, define $\langle C,P\rangle := \sum_{\boldsymbol{\ell}\in\operatorname{supp}(P)} P_{\boldsymbol{\ell}} C_{\boldsymbol{\ell}}$. 
%
%\sacomment{Need to define what $\boldsymbol{i}$ means.}
%
This is finite iff $P_{\boldsymbol{\ell}}=0$ whenever $C_{\boldsymbol{\ell}}=+\infty$. For PMFs $P,Q$ on the same set, define $\mathrm{KL}(P\|Q):=\sum_{\boldsymbol{i}\in\operatorname{supp}(P)}P_{\boldsymbol{\ell}}\log(P_{\boldsymbol{\ell}}/Q_{\boldsymbol{l}})$, and the \textbf{\textit{unnormalized}} entropy of $P$ by $H(P):=-\sum_{\boldsymbol{\ell}\in\operatorname{supp}(P)}P_{\boldsymbol{\ell}}(\log P_{\boldsymbol{\ell}}-1)$.

\paragraph{Problem Setting:} 
Let $\mathcal{D} = \{x_i\}_{i=1}^N$ be a dataset with $x_i\in\mathcal{X}$. 
%
% \saedit{In CL, for any $x_i$ referred to as an anchor, one selects a positive sample $x_j$, and a negative sample $x_k$ from the data set and aims to learn a $d$-dimensional representation of these samples such that positive pairs are closely aligned and negative pairs are misalinged. Specifically,}{} 
Let 
$\mathcal{Z}^\theta := \{z_i = f^\theta(x_i) \in \mathbb{S}^{d-1}, i = 1, \ldots, N\}$,
be the set of embeddings produced by an encoder $f^\theta$ in a representation space, which in this work we restrict to be the unit hypersphere, $\mathbb{S}^{d-1}$, i.e., for all $i$, $\|z_i\| = 1$. We first recall IOT-CL \cite{shi2023understanding} and highlight its limitations.

\subsection{IOT-CL: motivation and review}
% \sacomment{It will be good if we can also write this with $\mathcal{A}$ as the set of positive samples that are given. This will make it in line with the next section.}
The IOT-CL method proposed in \cite{shi2023understanding} defines a $N \times N$, non-negative, extended real-valued cost matrix
$C^\theta$, 
e.g., $C^\theta_{ij}=1-\langle f^\theta(x_i),f^\theta(x_j) \rangle$, for all distinct $i,j$ in $[N]$, and  $C^\theta_{ii} = +\infty$ for all $i \in [N]$, and proposes to learn the representation map $f^{\theta^*}$ via 
\begin{align} 
& \theta^*_{\text{IOT-CL}} = \arg\min_{\theta}  \mathrm{KL} \big(\tilde{P}\,\|\,P^{C^\theta}\big) \nonumber  \\ 
& \text{ s.t. } P^{C^\theta} = \argmin_{P\in \Pi(\mathbf{u}_N, \mathbf{u}_N)} \left[\langle C^\theta,P\rangle + \varepsilon\,\mathrm{KL}(P||\mathbf{u}_N \otimes \mathbf{u}_N )\right],
\label{eq:clot_final}
\end{align}
where $\varepsilon >0$ and $\tilde{P}$ is a ground-truth coupling PMF that is uniform over the set of distinct positive (similar) pairs in the dataset: $\forall i, j, \in [N]$,
$$
\tilde{P}_{ij} = 
\begin{cases}
\frac{1}{\gamma}, & \text{if $(x_i,x_j)$ form a positive pair and } i\neq j,\\
0, & \text{otherwise,}
\end{cases}
$$
with $\gamma$ equal to the total number of positive pairs in the dataset. 
Note that since $C^{\theta}_{ii} = +\infty$ for all $i\in[N]$, we must have $P_{ii} = 0$ if $\langle C , P \rangle$ is finite.
Since the marginals are uniform, $P^{C^\theta} = $ $\argmin_{P \in \Pi(\mathbf{u}_N, \mathbf{u}_N)} \big[\langle C^\theta,P\rangle + \varepsilon \sum_{(i,j) \in  \mathrm{support}(P)} P_{ij}\log P_{ij}\big] $ or 
equivalently, $P^{C^\theta} = $ 
$\argmin_{P \in \Pi(\mathbf{u}_N, \mathbf{u}_N)} \big[\langle C^\theta,P\rangle - \varepsilon H(P)]$, which makes the entropy-regularization explicit. 
Thus, the overall aim is to learn a mapping $f^{\theta}$ that best aligns the ground-truth coupling PMF $\tilde{P}$ with the coupling PMF $P^{C^\theta}$ obtained by solving an entropy-regularized OT problem with a cost capturing alignment of positive pairs in the representation space. 

\textbf{\textit{It is evident that $\theta^*_{\text{IOT-CL}}$ in (\ref{eq:clot_final}) only accounts for positive relations through $\tilde{P}$ and ignores the contrastive component}}, i.e. negative samples, that are the basis of contrast and has been shown to be very useful \cite{jiang2024hard, robinson2020contrastive}. Aligning positive pairs only is not enough for general downstream tasks, as it can lead to a degenerate solution of representations (for the arguments, see Appendix \ref{app: related_work}). This observation motivates us to extend the IOT-CL formulation via \textbf{\textit{explicitly incorporating negative sampling through Multi-Marginal OT (MMOT)}}, simultaneously pushing negative samples apart while pulling positive samples closer.

\subsection{Neg-MMIOT-CL:  Multi-Marginal Inverse OT for CL with negative samples} \label{sec:neg-mmot}
Here we are given a set of \textit{distinct} (anchor, positive, negative) triplets from the dataset. Let
$\mathcal{A} = \{(i,j,k) \in [N]^3$ distinct: $x_i$ is an anchor, $x_j$ is a positive sample and $x_k$ a negative sample\} denote the set of ``admissible'' triplets with (anchor, positive) interpreted as a positive pair and (anchor, negative) as a negative pair. We define the ground-truth coupling PMF $\tilde{P}$ to be uniform over the set of admissible triplets in the dataset, i.e., $\forall i,j,k \in [N]$,
$$\tilde{P}_{ijk} = 
\begin{cases}
\frac{1}{|\mathcal{A}|}, & \text{if } (i,j, k) \in \mathcal{A},\\
0, & \text{otherwise.}
\end{cases}
$$
We then define the third-order coupling cost tensor $C^\theta$ for \textbf{\textit{all}} triplets $(i,j,k) \in [N]^3$ as follows:
\begin{equation}
\small
\label{eq: c-theta}
C^\theta_{ijk} := 
\begin{cases}
\psi \left( \frac{\langle f^\theta(x_i), f^\theta(x_j) \rangle - \langle f^\theta(x_i), f^\theta(x_k) \rangle}{\tau}\right)  & \text{ $i,j,k$ distinct}, \\
+\infty & \text{otherwise}, 
\end{cases}
\end{equation}
where $\psi(t)$ is a real-valued cost-shaping function which is strictly decreasing, and $\tau > 0$ is the temperature parameter. Recall the notation $z = f^\theta(x)$.
Representations that increase $\langle z_i, z_j \rangle$ or decrease $\langle z_i, z_k \rangle$ will have a smaller coupling cost. We set $C^\theta_{ijk} = +\infty$ when $i,j,k$ are not all distinct to prevent self-coupling, as in IOT-CL. We define the Neg-MMIOT-CL representation mapping as 
\begin{align} 
& \theta^*_{\text{Neg-MMIOT-CL}} = \arg\min_{\theta}  \mathrm{KL} \big(\tilde{P}\,\|\,P^{C^\theta}\big) \nonumber \\
& \text{ s.t } P^{C^\theta} = \argmin_{P\in \Pi(\mathbf{u}_N, \mathbf{u}_N, \mathbf{u}_N)} \left[\langle C^\theta,P\rangle - \varepsilon\,H(P)\right].
\label{eq:neg_mmot_final}
\end{align} 
We note the following differences with IOT-CL.
\begin{enumerate}[leftmargin=*]
\setlength \itemsep{-0pt}
    \item The cost now \emph{contrasts} triplets $x_i, x_j, x_k$. These can all be positives, all negatives, or form a contrastive triplet, i.e. $x_i,x_j$ are positive pairs and $x_i,x_k$ are negative OR $x_i,x_j$ are negative pairs and $x_i,x_k$ are positive pairs. Some special cases of $\psi (t)$ are $\psi(t) = -t$, $\psi(t) = \log(1 + e^{-t})$.
    \item The pulling of positives and pushing of negatives is achieved via the $\mathrm{KL}$ objective that enforces the cost to be low on the support set of positive-negative triplets.
\end{enumerate}
A pseudo-code for solving the optimization problem is given in Algorithm~\ref{alg:unified_ot_cl}.

\subsection{Neg-IOT-CL-PushPull: a computationally scalable alternative to Neg-MMIOT-CL} \label{sec:neg-iot-pushpull}
A limitation of implementing Neg-MMIOT-CL lies in its computational cost: at each iteration, we must solve a transport problem using Sinkhorn over a $3$-dimensional tensor of size $N$, resulting in a complexity of $\mathcal{O}\!\left(N^3 \varepsilon^{-2}\right)$ \cite{piran2024contrasting}. This is quite prohibitive when $N$ is large. To address this issue, we propose a simplified variant of Neg-MMIOT-CL that decouples the matching and repelling mechanisms into two separate components. We refer to this formulation as Neg-IOT-CL-PushPull. In this model, we construct two sets of admissible \textbf{\textit{tuples}} by projecting $\mathcal{A}$ onto its positive- and negative-pair marginals, i.e.,
\begin{align*}
    \mathcal{A}^+ & =\{(i,j):  (i,j,k)\in\mathcal{A} \text{ for some } k\}, \\
    \mathcal{A}^- & =\{(i,k): (i,j,k)\in\mathcal{A}  \text{ for some } j\}.
\end{align*}
We define positive-pair and negative-pair ground-truth couplings, denoted by $\tilde{P}_{+}$ and $\tilde{P}_{-}$ respectively, as follows: $\forall i, j, k \in [N]$,
$$
(\tilde P_{+})_{ij}=\frac{\mathbf{1}\!\left((i,j)\in\mathcal{A}^+\right)}{|\mathcal{A^+}|},
\qquad
(\tilde P_{-})_{ik}=\frac{\mathbf{1}\!\left((i,k)\in\mathcal{A}^-\right)}{|\mathcal{A^-}|}.$$
Thus, $\tilde P_+$ places uniform mass over all positives and $\tilde P_-$ over all negatives globally.

\paragraph{Positive coupling.}
For positive pairs, we adopt the standard entropic OT formulation
\begin{equation}
P^{C^\theta}_{+}
  = 
  \argmin_{P\in\Pi(\mathbf{u}_N,\mathbf{u}_N)}
  \Big[\langle C^\theta\!,\,P\rangle
  -\varepsilon_{+}\,H(P)\Big],
\label{eq:pos_ot}
\end{equation}
where $\forall i, j \in [N]$, 
\[
C^\theta_{ij} := 
\begin{cases}
\psi\left(\frac{\langle f^\theta(x_i), f^\theta(x_j) \rangle}{\tau}\right) & \text{if } i \neq j, \\
+\infty & \text{otherwise},
\end{cases}
\]
with $\psi$ and $\tau$ as defined in the Neg-MMIOT-CL method. 
\paragraph{Negative coupling (anti-transport).}
To incorporate repulsion, we define an analogous coupling
that \emph{maximizes} the transport cost over negative pairs:
\begin{equation}
P^{C^\theta}_{-}
  =
  \argmax_{P\in\Pi(\mathbf{u}_N,\mathbf{u}_N)}
  \Big[\langle C^\theta\!,\,P\rangle
  +\varepsilon_{-}\,H(P)\Big].
  \label{eq:neg_ot}
\end{equation}
This formulation can be viewed as an \emph{anti-transport} problem that encourages large pairwise distances between features with negative relationship while maintaining entropy-controlled smoothness. We define the Neg-IOT-CL-PushPull representation mapping as
{\small
\begin{align}
\label{eq:iot_pushpull_final}
& \!\!\!\!\!\theta^*_{\text{Neg-IOT-CL-PushPull}}
= \argmin_{\theta} \Big[
\underbrace{\mathrm{KL}\!\big(\tilde{P}_{+}\,\|\,P^{C^\theta}_{+}\big)}_{\text{align positives}}
+\underbrace{\mathrm{KL}\!\big(\tilde{P}_{-}\,\|\,P^{C^\theta}_{-}\big)}_{\text{repel negatives}}
\Big] \nonumber \\
& \text{ s.t }
P^{C^\theta}_{+} \text{ satisfies } (\ref{eq:pos_ot}),
P^{C^\theta}_{-} \text{ satisfies } (\ref{eq:neg_ot}).
\end{align}
}
A pseudo-code for solving the optimization problem is given in Algorithm \ref{alg:unified_ot_cl}. 

\textbf{Remarks.} Although $\widetilde{P}_{-}$is uniform over admissible negatives, the learned anti-transport plan is similarity dependent $P_{-}^{C^\theta}=\operatorname{diag}(\mathbf{u}_N) \exp \left(C^\theta / \epsilon_{-}\right) \operatorname{diag}(\mathbf{u}_N)$ on the allowed support. Furthermore, $
\frac{\partial}{\partial C^\theta} \mathrm{KL}\left(\widetilde{P}_{-} \| P_{-}^{C^\theta}\right)=\frac{P_{-}^{C^\theta}-\widetilde{P}_{-}}{\epsilon_{-}}$. Thus the updates are not equal across negatives. A high-similarity, insufficiently separated negative has a small cost because $\psi$ is decreasing, receives too little anti-transport mass, and therefore has $P_{-}^{C^\theta}<\widetilde{P}_{-}$. Gradient descent increases its cost, which decreases its similarity. Conversely, already well-separated negatives receive a smaller or oppositely signed correction. Neg-IOT-CL-PushPull therefore performs adaptive, residual based hard-negative correction; the uniform target enforces coverage of all negatives rather than equal repulsive force. 

\subsection{Algorithms for Neg-MMIOT-CL and Neg-IOT-CL-PushPull}
Algorithm~\ref{alg:unified_ot_cl} describes the detailed steps of our implementation for solving Neg-MMIOT-CL and Neg-IOT-CL-PushPull in one mini-batch iteration. The \textsc{Sinkhorn} function refers to the standard Sinkhorn algorithm \cite{peyre2019computational} and \textsc{MM\text{-}Sinkhorn} is from \cite{piran2024contrasting}. Our work can be applied in both \textbf{supervised CL (SCL)} and \textbf{unsupervised CL (UCL)} settings. In SCL, a positive sample has the same label as the anchor whereas a negative sample's label differs from that of the anchor. In UCL, labels are unavailable, so different augmented views of the same instance are treated as positives and other instances as negatives. In terms of the admissible set $\mathcal A$, if $\{y_i\}$ denote the sample labels in SCL, then  $\mathcal A_{\mathrm{SCL}} =
\{(i,j,k) \text{ distinct}: \; y_i=y_j,\; y_i\neq y_k\}$. In UCL, if $r(i)$ denotes the original instance from which augmented view $i$ was generated, then $\mathcal A_{\mathrm{UCL}} =
\{(i,j,k) \text{ distinct}:\; r(i)=r(j),\; r(i)\neq r(k)\}$. For clarity, we provide pseudo code for SCL in Appendix \ref{app:pseudo_for_specific_setting}. The same 
formulation extends to UCL.

\begin{algorithm*}[t!]
%\caption{Unified OT-based contrastive training: Neg-MMIOT-CL and Neg-IOT-CL-PushPull per batch iteration}
\caption{Neg-MMIOT-CL and Neg-IOT-CL-PushPull batch training}
\label{alg:unified_ot_cl}
\scriptsize
\textbf{Input:}
batch of training samples $\{x_1,\ldots,x_B\}$,
admissible relations $\mathcal{A}$, 
encoder $f^\theta$, 
cost-shaping function $\psi$, 
regularization parameters $\varepsilon,\varepsilon_+,\varepsilon_-$, 
temperature $\tau$, 
Sinkhorn iteration $n_{it}$,
optimizer $\mathsf{Opt}$.

\begin{algorithmic}
\State For $i=1,\ldots,B$, compute normalized embeddings
    $z_i \gets \frac{f^\theta(x_i)}{\|f^\theta(x_i)\|}$
\end{algorithmic}
\vspace{2ex}
\noindent
\begin{minipage}[t]{0.45\linewidth}
\raggedright
\textbf{Neg-MMIOT-CL} 
\vspace{0.25em}

%\begin{enumerate}
    %\item 
    \noindent 1.~Define target triplet coupling: $\forall i,j,k \in [B]$, 
    %supported on $\mathcal{A}$:
    \[
     \tilde P_{ijk}
    =
    \frac{
     \mathbf{1}\!\left((i,j,k)\in\mathcal{A}\right)
    }{
      |\mathcal{A}|
    }.
    \]

    %\item 
    \noindent 2.~Construct triplet cost tensor: $\forall i,j,k \in [B]$, 
    \[
    \!\!\!C^\theta_{ijk}
    =
    \begin{cases}
    \psi\Big(\frac{\langle z_i,  z_j\rangle - \langle z_i,  z_k\rangle}{\tau}\Big), &  i,j,k, \text{ distinct}\\[0.4em]
    +\infty, & \text{otherwise.}
    \end{cases}
    \]

    %\item 
    \noindent 3.~Solve entropic MMOT problem (\ref{eq:neg_mmot_final}):
    \[
        P^{C^\theta} \;\gets\; \textsc{MM\text{-}Sinkhorn}(C^\theta,\varepsilon, n_{it})
        \]
    %\item 
    \noindent 4.~Compute loss:
    \[
    \mathcal{L}(\theta)
    =
    \mathrm{KL}\!\left(\tilde P\,\|\,P^{C^\theta}\right)
    \]
%\end{enumerate}
\end{minipage}
%\hfill 
\hspace{2ex}
%\qquad
\begin{minipage}[t]{0.48\linewidth}
\raggedright
\textbf{Neg-IOT-CL-PushPull} 
\vspace{0.25em}

%\begin{enumerate}
    %\item 
    \noindent 1.~Project $\mathcal{A}$ onto positive- and negative-pairs sets:
    \[
    \mathcal{A}^+
    =
    \{(i,j): \exists k \text{ s.t. } (i,j,k)\in\mathcal{A}\},
    \]
    \[
    \mathcal{A}^-
    =
    \{(i,k): \exists j \text{ s.t. } (i,j,k)\in\mathcal{A}\}.
    \]

    \noindent 2.~Define target pairwise couplings
    \[
    (\tilde P_{+})_{ij}
    =
    \frac{
      \mathbf{1}\!\left((i,j)\in\mathcal{A}^+\right)
    }{
      |\mathcal{A^+}|
    },
    %\]
    %\[
    (\tilde P_{-})_{ik}
    =
    \frac{
      \mathbf{1}\!\left((i,k)\in\mathcal{A}^-\right)
    }{
      |\mathcal{A^-}|
    }.
    \]
    %\item 
    \noindent 3.~Build the pairwise cost matrix: $\forall i, j \in [B]$,
    \[
    C^\theta_{ij}
    =
    \begin{cases}
    \psi\!\big( \frac{\langle z_i,z_j \rangle}{\tau} \big), & i\neq j,\\[0.4em]
    +\infty & \text{otherwise}.
    \end{cases}
    \]

    %\item 
    \noindent 4.~Solve positive and negative OT problems (\ref{eq:pos_ot}), (\ref{eq:neg_ot}): 
    %\ref{eq:iot_pushpull_final}
    % $$ 
    % (M_+)_{ij} \!=\! \mathbf{1}\!\left((i,j)\in\mathcal{A}^+\right)\quad
    % (M_-)_{ik} \!=\! \mathbf{1}\!\left((i,k)\in\mathcal{A}^-\right)
    % $$
    $$P^{C^\theta}_{+} \gets \textsc{Sinkhorn}\!\left(C^\theta,\ \varepsilon_{+},\ n_{it}\right); P^{C^\theta}_{-} \gets \textsc{Sinkhorn}\!\left(-C^\theta,\ \varepsilon_{-},\ n_{it}\right)$$
    %\item 
    \noindent 5.~Compute loss:
    \[
    \mathcal{L}(\theta)
    =
    \mathrm{KL}\!\left(\tilde P_{+}\,\|\,P^{C^\theta}_{+}\right)
    +
    \mathrm{KL}\!\left(\tilde P_{-}\,\|\,P^{C^\theta}_{-}\right),
    \]
%\end{enumerate}
\end{minipage}
\vspace{0.2em}
\begin{algorithmic}
    \State Compute gradient $\nabla_\theta \mathcal{L}(\theta)$
    \State Update parameters $\theta \gets \mathsf{Opt}\!\left(\theta,\nabla_\theta \mathcal{L}(\theta)\right)$
\end{algorithmic}
\end{algorithm*}

\section{Optimal Representation Geometry Analysis} \label{sec: analysis}
% \sacomment{Need to wrap this section in $\leq 2$ pages.}

In this section, we analyze the geometric structure of the optimal embedding features induced by the Neg-MMIOT-CL method. We consider the \textit{Unconstrained Features Model} (UFM) with admissible triplets comprising the anchor and positive samples in the same class and the negative sample in a different class. When the \textit{classes are balanced}, {\textit{$\psi$ is decreasing and affine}} and the \textit{dimension of the representation space is at least the number of classes minus one}, then Neural Collapse (NC) will occur and the simplex Equiangular Tight Frame (ETF) is the optimal geometry of the embedding features. Due to space constraints, proofs of all results are presented in Appendix~\ref{appendix:proof}.

\noindent\textbf{A1: Balanced Class-Structure.} 
We consider a dataset of $N=bK$ samples that consists of $K$ disjoint classes, each having $b$ samples. The admissible set is defined by:
\[
\mathcal{A}=\{(i,j,k): i,j,k \textit{ are distinct with } \ y_i=y_j,\ y_i\neq y_k\},
\]where $y_i,y_j,y_k$ denote the labels of samples $i,j,k$, respectively. 
For each anchor $i\in [N]$, there are $(b-1)$ choices for positive samples (choices for index $j$) from the same class as $i$ and $b(K-1)$ choices for negative samples (choices for index $k$) from other classes. Thus, $|\mathcal{A}| = N(b-1)b(K-1)=N(b-1)(N-b)$ and 
$$
\tilde{P}_{i j k}:= \begin{cases}\frac{1}{|\mathcal{A}|} = \frac{1}{N(b-1)(N-b)}, &\text{if } (i, j, k) \in \mathcal{A}, \\ 0, & \text {otherwise. }\end{cases}
$$
Note that all three marginal PMFs of this ground-truth coupling $\tilde{P}$ are equal to $\mathbf{u}_N$, i.e., $\tilde{P} \in \Pi(\mathbf{u}_N,\mathbf{u}_N,\mathbf{u}_N)$. Also note that $\text{support}(\tilde{P}) \subseteq [N]^3_{\neq}$ since if $(i,j,k) \in \mathcal{A}$, then $i, j, k$ must all be distinct.

\noindent\textbf{A2: Unconstrained Features Model (UFM).} In practice, the family of representation functions $\{f^\theta\}$ is constrained to be representable by a neural network having a specific architecture. For theoretical analysis we assume that the representation capacity of the neural network is sufficiently large to  approximate an arbitrary mapping $f^{\theta}$ to any desired accuracy. This assumption is used in several previous works, e.g.,  \cite{jiang2024hard,nguyen2024neural, shi2023understanding, shi2024ot} which treat a neural network’s final-layer feature vectors, denoted by $z = f^{\theta}(x)$, as the free optimization variables instead of the network weights $\theta$. This \textbf{\textit{decouples}} feature geometry from the complex nonlinear encoder weight parameterization.

\noindent\textbf{A3: Dimension of the representation space \textit{vs.}~number of classes.} We also assume that the dimension of the representation space is at least the number of classes minus one, \textit{i.e.}, $d \geq K-1$. This condition is sufficient to ensure that the optimal representations form an ETF in the representation space. This assumption has been used in prior works \cite{papyan2020prevalence,jiang2024hard,nguyen2024neural}.

The following lemma proves that the KL-divergence of the ground-truth coupling from the entropic-OT coupling is a convex function of the entropic-OT cost tensor. 
\begin{lemma}\label{lem:convexity_to_cost}
Let $\mathcal{C}$ be the convex set of all third-order tensors with $C_{ijk}$ finite for all $(i,j,k) \in [N]^3_{\neq}$ and 
$C_{ijk} = +\infty$ otherwise. Let $P^C := \argmin_{P\in\Pi(\mathbf{u}_N,\mathbf{u}_N,\mathbf{u}_N)} \Big[\langle C, P \rangle - \epsilon H(P)\Big]$ (the minimizer exists and is unique). Then for $\tilde{P}$  in Assumption~A1, $g(C) := \mathrm{KL}(\tilde{P}||P^C)$ with $C \in \mathcal{C}$, is a convex function of $C$. 
\end{lemma}

\noindent \textbf{Remark:} In Lemma~\ref{lem:convexity_to_cost} there are no additional requirements on the form of the cost tensor, such as  $C_{ijk} = \psi \left( \frac{\langle f^\theta(x_i), f^\theta(x_j) \rangle - \langle f^\theta(x_i), f^\theta(x_k) \rangle}{\tau}\right)$, nor on the dimension of the representation map $f^\theta(\cdot)$.

Next, the convexity of the objective function in terms of $C$ together with the balanced-class assumption allows us to show that any optimal $C$ must satisfy an invariance property.

\begin{lemma}\label{lem:achieve_sym}
    Let $\mathcal{C}$, $P^C$, and $g(\cdot)$ be as in Lemma~\ref{lem:convexity_to_cost}. 
    Let $\pi: [N] \rightarrow [N]$ denote a bijection. For any order $3$ tensor $T$, let $T \circ \pi$ be the order $3$ tensor where for all $i, j, k \in [N]$, $(T\circ \pi)_{i j k} := T_{\pi(i) \pi(j) \pi(k)}$. Let
    $
    \Gamma:=\{\pi:\ y_{\pi(i)}=y_{\pi(j)}\iff y_i=y_j\ \forall i,j\}
    $
    denote the group of all permutations of sample indices within each class and permutations of class labels and
   $
    \bar{C} := \tfrac{1}{|\Gamma|} \sum_{\pi \in \Gamma} C\circ \pi.
   $ 
    Then for all $C \in \mathcal{C}$ and all $\pi \in \Gamma$, $\bar{C} \circ \pi = \bar{C}$ and $g(\bar{C}) \leq g(C).$
\end{lemma}

\noindent\textbf{Remarks:} The implication of Lemma~\ref{lem:achieve_sym} is that in order to minimize $g(C)$ over $\mathcal{C}$, it is sufficient to only consider cost tensors $C \in \mathcal{C}$ that are invariant under the permutations in $\Gamma$. We note that only Assumption~A1 has been utilized in Lemma~\ref{lem:convexity_to_cost} and Lemma~\ref{lem:achieve_sym}.

Next, we use the invariance property of the optimal solution in Lemma \ref{lem:achieve_sym} to show that the optimal representation features must satisfy a condition called \textit{two-distance class-homogeneous configuration}. 
% \tncomment{I just feel should we not use the "unconstrained" term in Lemma 3.3.? they may ask what does it mean? Also, should we say $d\geq N$? instead of $d=N$? }
\begin{lemma}\label{lem:achieve_homo}
    If $\psi$ is affine, and the embedding dimension $d \geq N$, then there exist optimal cost tensors having form (\ref{eq: c-theta}) whose embeddings satisfy the \textit{two-distance class-homogeneous configuration.}
    Formally, there exist optimal representations $z_i^*$ and $z_j^*$ with corresponding labels $y_i$ and $y_j$ such that
    \[
\langle z_i^*,z_j^*\rangle=
\begin{cases}
\alpha, & y_i=y_j,\\
\beta, & y_i\neq y_j
\end{cases}
\]
where $\alpha$ and $\beta$ are two constants, $\alpha,\beta \in [-1,1]$. 
\end{lemma}

\paragraph{Remarks. } Lemma \ref{lem:achieve_homo} is mainly based on the invariance of the optimal cost tensor $C^*$ under permutations in $\Gamma$ when $\psi$ is affine and classes are balanced.  Although the lemma assumes $\psi$ to be affine and strictly decreasing, the proof only requires $\psi$ to be affine and non-increasing. Note that even with an affine $\psi$, Lemma~\ref{lem:achieve_homo} is non-trivial. The most recent OT-based contrastive learning method closely related to our work \cite{shi2023understanding},
only considers $\psi(t) = -t$, but provides no theoretical analysis.

\begin{theorem}
\label{thm:nc-neg-mmot}
If $\psi (t) = -t$, then there exist optimal representation features of the Neg-MMIOT-CL objective in (\ref{eq:neg_mmot_final}) that exhibit Neural Collapse, \textit{i.e.}, all samples in each class collapse to their class mean vector and the class mean vectors form a simplex ETF. Formally, let $\mu_{i}=\mathbf{E}[f(x_j)| y(x_j)=i]=\mathbf{E}[z_j| y(x_j)=i]$ denote the class mean vector of class $i$ in representation space. Then there exist optimal solutions of (\ref{eq:neg_mmot_final}) which satisfy
\begin{enumerate}[leftmargin=*]
    \item Class collapse:  $f(x_j)=\mu_{i}, \forall j$ such that $y(x_j)=i$. 
    \item ETF configurations: (a) $\|\mu_i\|=1, \forall i$, (b) $\sum_{i=1}^K \mu_i=0$, and (c) $\langle \mu_i,\mu_j \rangle = -\frac{1}{K-1}, \forall i \neq j$.
\end{enumerate}
The rank of an optimal cost tensor with the embeddings satisfying Neural Collapse is $K-1$ and thus it is enough to restrict the embedding dimension to any $d \geq K-1$.
\end{theorem}

\paragraph{Remarks.} The condition $d\geq K-1$ (A3) is not imposed, but emerges from the proof of Theorem \ref{thm:nc-neg-mmot}. Indeed, this condition arises because a simplex ETF with $K$ vertices must lie in a dimension of at least $K-1$. If this condition is violated, \textit{i.e.}, if $d< K-1$, then it is impossible to find $K$ class mean vectors $\mu_1,\mu_2,\dots,\mu_K$ satisfying all the ETF conditions in Theorem \ref{thm:nc-neg-mmot}, specially, the last condition where one requires that $\langle \mu_i,\mu_j \rangle = -\frac{1}{K-1}, \forall i \neq j$. 

Proving that NC/ETF is a global optimizer for a broader class of $\psi$ remains an open but promising research direction. Proposition \ref{prop: stationary} shows that, when $\psi$ is differentiable, the NC/ETF configuration is indeed a stationary point not only of Neg-MMIOT-CL, but also of Neg-IOT-CL-PushPull.

\begin{proposition} \label{prop: stationary}
Assume $\psi$ is differentiable. Then any NC configuration whose class means form a simplex ETF is stationary for the spherical gradient flow $\frac{d}{dt} z_i=-P_{z_i}^{\perp}\nabla_{z_i}g(C(z)), i\in[N]$ where $P_{z}^{\perp}$ denotes the orthogonal projection at $z$, i.e. $P_{z}^{\perp}(x) = x - (x^\top z)z$.
That is, if $z_i=\mu_{i}$, $\|\mu_i\|=1$, $\sum_{i=1}^K\mu_i=0$, and
$\langle \mu_i,\mu_j\rangle=-1/(K-1)$ for $i\neq j$, then
\[
P_{z_i}^{\perp}\nabla_{z_i}g(C(z))=0
\qquad
\forall i\in[N].
\]
Similarly, for the PushPull counterpart,
$g_{\mathrm{PP}} (C)=\mathrm{KL}\left(\widetilde{P}_{+} \| P_{+}^C\right)+\mathrm{KL}\left(\widetilde{P}_{-} \| P_{-}^C\right)$
, we also have 
\[
P_{z_i}^{\perp}\nabla_{z_i}g_{\mathrm{PP}}(C(z))=0
\qquad
\forall i\in[N].
\]
\end{proposition}

\section{Experiments} \label{sec: exp}
We evaluate the proposed algorithms on three kinds of experiments: (1) Synthetic Gaussian mixtures, where we test whether Neg-MMIOT-CL and Neg-IOT-CL-PushPull exhibit Neural Collapse geometry and mitigate dimensional collapse,  (2) Vision benchmarks, where we assess improvements in supervised and unsupervised representations, and (3) CLIP-style pretraining, where we evaluate gains in multimodal retrieval and zero-shot transfer. Also, we conduct an ablation study to test the efficiency of Neg-IOT-CL-PushPull in Appendix~\ref{sec:ablation}. All experiments are run on NVIDIA H200 GPUs.
\vspace{-0.1in}
\paragraph{Detailed Experimental Setup:} Detailed settings and hyperparameter configurations of all the experiments are described Appendix~\ref{app:exp_setup}.

\subsection{Synthetic Gaussian mixture: neural collapse and dimensional collapse}
\begin{figure*}[htp!]
    \centering
    \includegraphics[width=0.85\linewidth]{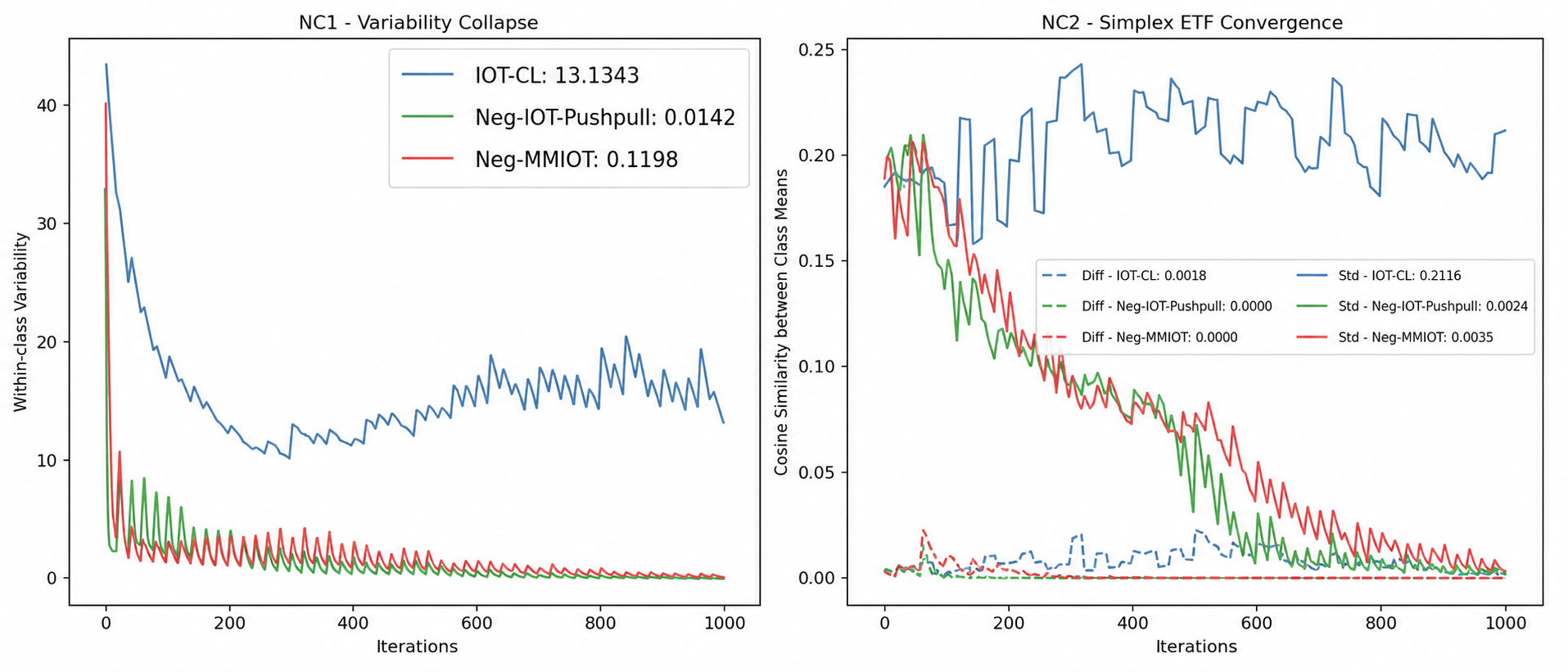}
    \caption{Testing Neural Collapse: NC1 measures within-class variability with smaller values indicating stronger class collapse. NC2 measures how close the class means are to the simplex-ETF geometry. Both NC2 quantities should approach zero when the class means form a simplex ETF. The numbers in the legends are the NC metrics in the last iteration.}
    \label{fig:nc_synthetic_affine}
    \vspace{-0.2in}
\end{figure*}
\begin{figure*}[htp!]
    \centering
    \subfloat{\includegraphics[width=0.33\linewidth]{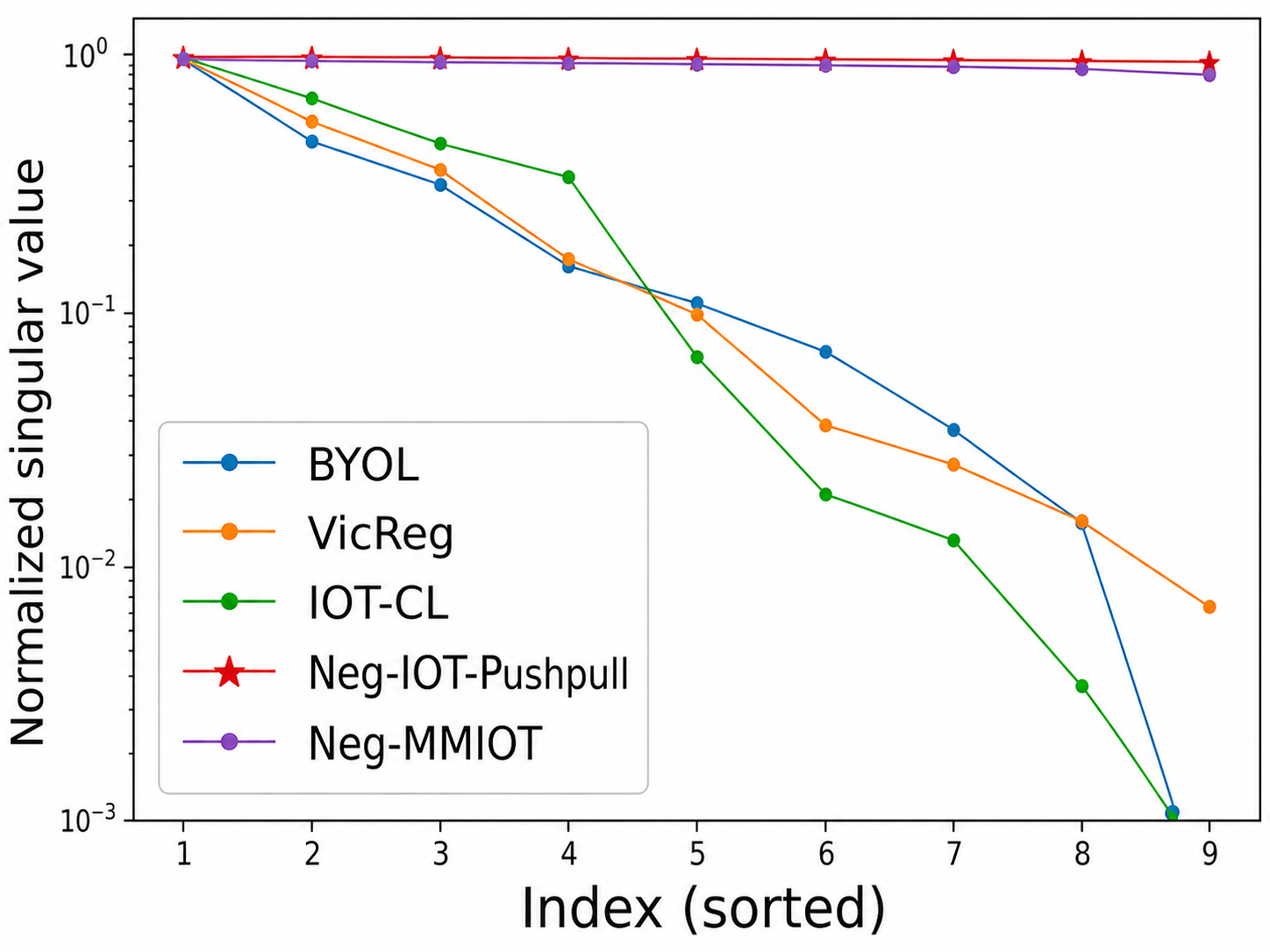}\label{rate}}
    \hfill
    \subfloat{\includegraphics[width=0.33\linewidth]{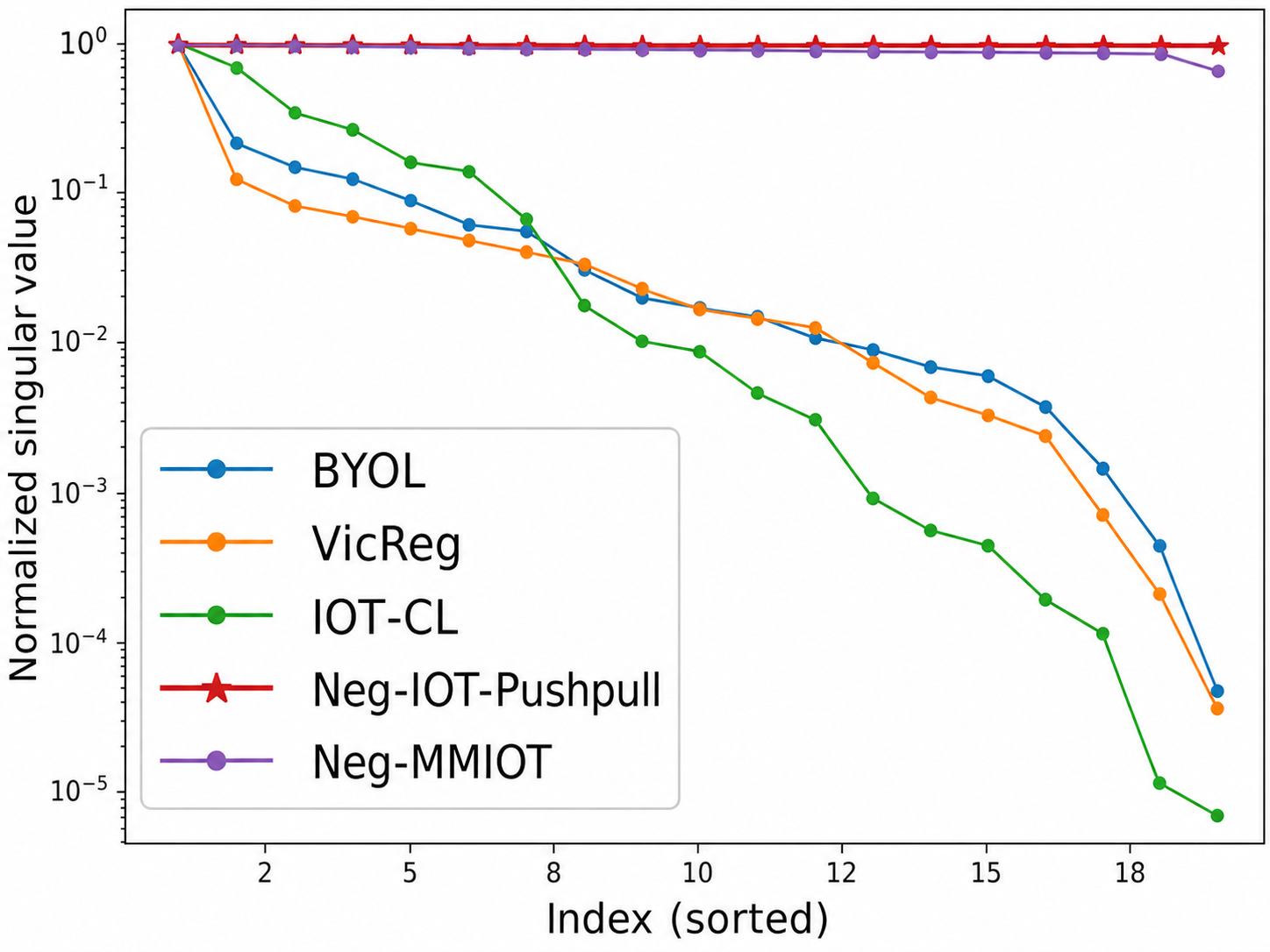}\label{rate}}
    \subfloat{\includegraphics[width=0.33\linewidth]{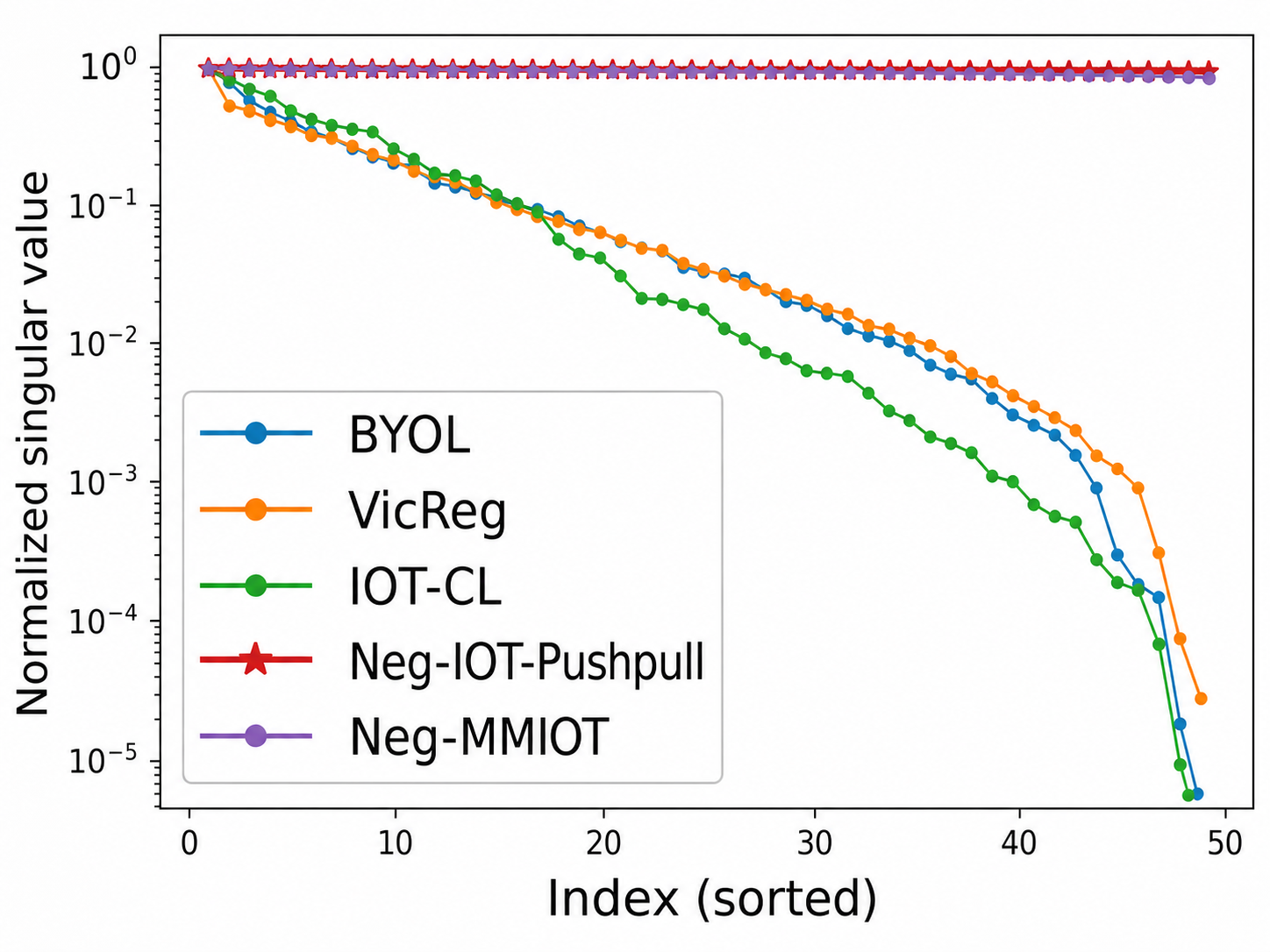}\label{rate}}\\
    \hspace{9ex} $10$ classes \hspace{19ex} $20$ classes \hspace{19ex} $50$ classes \hspace{5ex}
    \caption{Testing Dimensional Collapse.
    We plot the sorted (max)-normalized eigenvalues of the centered class-mean covariance matrix at the final epoch. A rapidly decaying spectrum indicates that the class means occupy only a few effective directions, i.e., dimensional collapse.}
    \label{fig:dc_synthetic_affine}
\end{figure*}
We study a controlled synthetic setup (detailed in Appendix \ref{app:exp_setup}) to test the theoretical results of \Cref{sec: analysis}, focusing on whether negative samples in Neg-MMIOT-CL and Neg-IOT-CL-PushPull induce class collapse, simplex geometry, and mitigate dimensional collapse. We compare Neg-IOT-CL-PushPull and Neg-MMIOT-CL with IOT-CL \cite{shi2023understanding}, BYOL, and VICReg. Evaluation criteria defined in Appendix~\ref{app:exp_setup} include \emph{Neural Collapse metrics (NC1, NC2)} and the spectrum of the class-mean covariance to assess \emph{dimensional collapse}. Appendix~\ref{app:add_exp_results} also provides a t-SNE visualization of representation features on a circle in $\mathbb{R}^2$.

Figures \ref{fig:nc_synthetic_affine} and \ref{fig:dc_synthetic_affine} support the theoretical results in Theorem~\ref{thm:nc-neg-mmot} for Neg-MMIOT-CL. Both Neg-IOT-CL-PushPull and Neg-MMIOT-CL drive within-class variability down much more aggressively than IOT-CL, and both move the class means toward the simplex-ETF target. At the same time, their class-mean covariance spectra remain substantially richer than those of the non-contrastive baselines, indicating better resistance to dimensional collapse. Overall, the synthetic results show that introducing explicit negative samples improves not only alignment within each class but also the global geometry of the learned representations.

\subsection{Vision benchmarks: supervised and unsupervised CL} 
We consider both supervised CL (SCL) and unsupervised CL (UCL) on MNIST \cite{lecun2002gradient}, SVHN \cite{netzer2011reading}, CIFAR-10, CIFAR-100 \cite{krizhevsky2009learning}, and TinyImageNet \cite{le2015tiny}, using ResNet-18, ResNet-34, ResNet-50 and ViT-B/16. These datasets consist of 28 x 28 x 1 images in 10 classes (MNIST), 32 × 32 × 3 images across 10 classes (SVHN, CIFAR-10), 100 classes (CIFAR-100), and 64 x 64 x 3 images across 200 classes (TinyImageNet), respectively. For image augmentation for UCL, we adopt a SimCLR-style two-crop pipeline, where each sample is transformed into two views using random resized cropping, horizontal flipping, color jittering, random grayscale, optional Gaussian blur, and normalization. More details are in Appendix \ref{app:exp_setup}. We report both linear-probe and k-NN accuracy in order to evaluate representation quality both with and without an additional trained classifier. 
%The results are averaged and reported in Table \ref{tab:scl_ucl}.

\begin{table*}[htp!]
\centering
\captionsetup{skip=5pt}
\renewcommand{\arraystretch}{1.35}
\setlength{\tabcolsep}{3pt}
\resizebox{0.9\linewidth}{!}{%
\begin{tabular}{|c|l|cc|cc|cc|cc|cc|}
\hline
{\textbf{Setting}} 
& \multicolumn{1}{c|}{\textbf{Method}}
& \multicolumn{2}{c|}{\textbf{MNIST}} 
& \multicolumn{2}{c|}{\textbf{SVHN}} 
& \multicolumn{2}{c|}{\textbf{CIFAR-10}} 
& \multicolumn{2}{c|}{\textbf{CIFAR-100}} 
& \multicolumn{2}{c|}{\textbf{Tiny-ImageNet}} \\ \cline{3-12}
&
& ResNet50 & ViT-B/16
& ResNet50 & ViT-B/16
& ResNet50 & ViT-B/16
& ResNet50 & ViT-B/16
& ResNet50 & ViT-B/16 \\
\hline

{\textbf{SCL}}
& InfoNCE
& 99.36 & 99.49
& 91.96 & 94.07
& \textbf{91.39} & 90.40
& 74.01 & 73.59
& 62.60 & 46.32 \\
\cline{2-12}

& InvaSpread
& 98.98 & 99.61
& 81.79 & 85.35
& 78.46 & 72.62
& 63.26 & 67.08
& 55.06 & 44.19 \\
\cline{2-12}

& Standard OT
& 99.21 & 99.60
& 92.08 & 94.24
& 90.78 & 84.72
& 72.71 & 73.49
& 61.20 & 45.27 \\
\cline{2-12}

& Neg-IOT-CL-PushPull (Ours)
& 99.30 & 99.64
& 91.89 & 94.38
& 89.59 & 88.69
& 74.80 & \textbf{74.67}
& 62.04 & 46.54 \\
\cline{2-12}

& Neg-MMIOT-CL (Ours)
& \textbf{99.41} & \textbf{99.66}
& \textbf{92.74} & \textbf{94.87}
& 90.64 & \textbf{91.97}
& \textbf{75.17} & 74.62
& \textbf{63.54} & \textbf{48.02} \\
\hline

{\textbf{UCL}}
& InfoNCE
& \textbf{97.73} & $98.40$
& 78.13 & $84.67$
& 76.81 & 84.57
& 49.05 & 56.99
& 55.23 & 51.45 \\
\cline{2-12}

& InvaSpread
& 97.65 & $97.95$
& 73.48 & $75.13$
& 64.88 & 64.13
& 44.51 & 44.17
& 46.85 & 43.18 \\
\cline{2-12}

& Standard OT
& 97.20 & $98.35$
& 75.96 & $83.94$
& 78.12 & 83.22
& 45.54 & 55.46
& 55.29 & 53.12 \\
\cline{2-12}

& Neg-IOT-CL-PushPull (Ours)
& 96.56 & $98.61$
& 78.63 & $87.45$
& \textbf{78.18} & \textbf{87.56}
& \textbf{52.20} & 55.45
& 55.37 & 54.35 \\
\cline{2-12}

& Neg-MMIOT-CL (Ours)
& \textbf{97.73} & $\mathbf{98.72}$
& \textbf{84.65} & $\mathbf{88.61}$
& 77.01 & 86.15
& 49.27 & \textbf{57.52}
& \textbf{60.27} & \textbf{57.48} \\
\hline
\end{tabular}%
}
\caption{Average linear-probe accuracy (\%) for SCL and UCL across datasets using ResNet50 and ViT-B/16 backbones. Bold values indicate the best result within each setting, dataset, and backbone.}
% \vspace{-0.2in}
\label{tab:scl_ucl}
\end{table*}

Due to the space limit, Table \ref{tab:scl_ucl} only shows average linear-probe results of ResNet-50 and ViT-B/16 through 4 different seeds; see Appendix \ref{app:detail_result_vision} for \textit{detailed results on different backbones with error bars}. The key message of these tables is not just that our method improves accuracy, but that explicit negative transport turns OT from a positive-matching objective into a genuinely contrastive learner. Relative to IOT-CL, the gains appear across both linear-probe and k-NN evaluation, which suggests better representation geometry rather than a mere classifier effect. In several settings the proposed objectives close and surpass InfoNCE; Neg-MMIOT-CL is typically the strongest full formulation, while Neg-IOT-CL-PushPull retains much of the same benefits with lower computational cost.

\vspace{-0.1in}
\subsection{Vision-Language benchmark: CLIP training}
Appendix~\ref{app:exp_setup} provides the background for CLIP and the detailed experimental setup. We benchmark CLIP trained on MS-COCO \cite{lin2014microsoft} on: \textbf{(1)} MS-COCO retrieval (image$\leftrightarrow$text) and \textbf{(2)} zero-shot classification on CIFAR-10/100 using the image encoder (i.e. classify images without training on CIFAR-10/100). Baselines include InfoNCE and OT variants from \cite{shi2024ot}: Standard OT (with/without uniform penalty), DBOT, Fused-Gromov OT, using their reported hyperparameters. The results are shown in Table~\ref{tab:clip}.

\begin{table*}[htp!]
\centering
\tiny
\captionsetup{skip=12pt}
\renewcommand{\arraystretch}{1.5}
\resizebox{0.9\linewidth}{!}{%
\begin{tabular}{|c|cc|cc|cc|cc|}
\hline
\multicolumn{1}{|c|}{\textbf{CLIP-Loss}}
& \multicolumn{2}{|c|}{\textbf{Image $\rightarrow$ Text}}
& \multicolumn{2}{|c|}{\textbf{Text $\rightarrow$ Image}}
& \multicolumn{2}{|c|}{\textbf{CIFAR-10}}
& \multicolumn{2}{|c|}{\textbf{CIFAR-100}} \\
\cline{2-9}
& Top-1 & Top-5
& Top-1 & Top-5
& Top-1 & Top-5
& Top-1 & Top-5 \\
\hline

infoNCE
& 67.22 & 92.51
& 65.91 & \textbf{94.61}
& 23.37 & 74.82
& 6.36 & 27.59 \\

\hline

Standard OT
& 43.92 & 81.86
& 30.90 & 70.54
& 26.62 & 78.56
& 4.77 & 22.95 \\

Standard OT - Uniform
& 64.25 & 92.40
& 63.22 & 92.06
& 25.97 & 78.67
& \textbf{6.51} & 27.68 \\

DBOT
& 15.89 & 46.68
& 15.52 & 45.92
& 24.98 & 70.36
& 5.72 & 27.19 \\

Fused-Gromov
& 14.32 & 43.92
& 13.94 & 42.87
& 24.99 & 76.06
& 3.72 & 19.40 \\

% UFG-OT
% & 7.9 & 39.8
% & 7.8 & 39.5
% & 19.22 & 59.25
% & 1.67 & 10.21 \\

\hline

Neg-IOT-CL-PushPull (Ours)
& 61.05 & 91.26
& \textbf{67.64} & 89.60
& 24.92 & 72.00
& 5.61 & \textbf{29.39} \\

Neg-MMIOT-CL (Ours)
& \textbf{69.28} & \textbf{94.76}
& 65.66 & 92.22
& \textbf{27.71} & \textbf{79.17}
& 5.02 & 19.69 \\

\hline
\end{tabular}%
}
\caption{MS-COCO retrieval and zero-shot transfer results (\%). Bold values indicate the best result}
\vspace{-0.2in}
\label{tab:clip}
\end{table*}

Table \ref{tab:clip} makes the same point in the multimodal setting: positive alignment alone is not enough for CLIP, and OT baselines without explicit negative transport can substantially weaken retrieval. By building repulsion into the transport objective, our methods recover strong cross-modal discrimination; Neg-MMIOT-CL achieves the best image-to-text retrieval and the strongest CIFAR-10 zero-shot transfer, while Neg-IOT-CL-PushPull remains competitive and even attains the best text-to-image top-1 and CIFAR-100 top-5. This shows that negative transport is not only useful in unimodal representation learning, but is a key ingredient for making OT effective in multimodal contrastive training.
\vspace{-0.1in}
\section{Conclusion}
\vspace{-0.1in}
We proposed Neg-MMIOT-CL, an entropy-regularized multi-marginal OT framework that explicitly incorporates negative separation into contrastive learning, along with the scalable pairwise variant Neg-IOT-CL-PushPull. We prove that, under balanced data, affine decreasing cost-shaping function, and sufficient embedding dimension, Neg-MMIOT-CL recovers Neural Collapse and simplex ETF geometry. Empirically, our methods improve geometry and downstream performance across synthetic, vision, and vision-language tasks. A limitation of the current work is that the theoretical proof is presently established only for the case where $\psi$ is affine; extending it to the general case remains an open direction for future work. Other important directions are to extend the proposed guarantees and algorithms to imbalanced data, finite-capacity networks, and larger-scale training settings.

\vspace{-0.1in}
\section{Acknowledgment}
\vspace{-0.1in}
This work is supported via NSF under Cooperative Agreement PHY-2019786 and PHY-2525568. Shuchin Aeron would also like to acknowledge funding via NSF DMS 2309519.

\bibliography{sample}
\bibliographystyle{unsrt}

\newpage
\appendix

\tableofcontents  
\addtocontents{toc}{\protect\setcounter{tocdepth}{1}}
\appendix
\numberwithin{equation}{section}
\numberwithin{theorem}{section}
\numberwithin{algorithm}{section}
\numberwithin{figure}{section}
\numberwithin{table}{section}

\section{Related work and further discussion} \label{app: related_work}
\begin{enumerate}
    \item \textbf{Contrastive Learning with Negative Samples:} Contrastive learning (CL) is typically built on objectives that align positive pairs while repelling negatives. This template underlies InfoNCE-style methods such as SimCLR \cite{chen2020simple}, supervised contrastive learning \cite{khosla2020supervised}, and remains the dominant formulation in both unimodal and multimodal settings \cite{radford2021learning}. A related line strengthens the repulsive term through hard-negative sampling, showing that the treatment of negatives can substantially affect both optimization and the resulting representation geometry \cite{jiang2024hard}. Our work is closest in spirit to this literature: like standard CL, it treats negative separation as essential, but it encodes attraction and repulsion through transport couplings over admissible tuples rather than through a pairwise softmax over batch similarities.
    \item \textbf{Non-Contrastive Learning with only Positive Samples:} In parallel, non-contrastive representation learning methods show that useful embeddings can be learned from positive pairs alone. BYOL \cite{grill2020bootstrap} avoids collapse through asymmetry between online and target networks, while VICReg \cite{bardes2021vicreg} replaces explicit negatives with variance and covariance regularization. These methods are relevant here because they illustrate a different route to alignment, where separation is induced only indirectly rather than by explicitly modeling negative relations.
    \item \textbf{Optimal Transport-based Contrastive Learning:} More recently, OT has been used to measure discrepancy and enforce alignment in representation learning. Examples include connection of contrastive loss with matching problem in \cite{chen2024your}; partial OT for CL in imbalanced multi-view clustering \cite{xue2025protocol}. Closest to our work, \cite{shi2023understanding} recast InfoNCE as an \emph{inverse} OT problem, replacing the usual anchor-centric view with a global batch-level matching perspective and showing that both InfoNCE training and softmax inference arise under different OT constraints. This connection motivates our approach. In vision--language learning, OT-based alignment has also improved CLIP-style models in challenging zero-shot settings \cite{shi2024ot}. \cite{piran2024contrasting} also extends this idea to multi-marginal matching-gap methods for multiple views or modalities.
    \item \textbf{Optimal geometry of Contrastive Learning:} A complementary line of work studies the representation geometry induced by CL objectives. For supervised contrastive learning, early analyses showed that global optima exhibit class collapse and simplex Equiangular Tight Frame (ETF) structure, and later works connected this picture to the information bottleneck, studied how prototypes can engineer geometry \cite{gill2024engineering, behnia2024supervised, kini2023symmetric}, and characterized both balanced \cite{graf2021dissecting} and imbalanced \cite{nguyen2024neural} regimes under unconstrained-feature models. Hard-negative sampling further emphasizes the role of explicit repulsion by recovering Neural Collapse while mitigating dimensional collapse \cite{jiang2024hard}, and recent analyses of sigmoid-based CL \cite{lee2024analysis,bangachev2025global} study temperature-dependent optimal structures and, with trainable temperature and bias, broader classes of global minimizers. To the best of our knowledge, \emph{comparable geometric results have not been established for OT-based Contrastive Learning objectives}. Our work fills this gap by providing such a geometry result in the OT/MMOT setting.
\end{enumerate}

\paragraph{Why Positive-Pair Alignment Alone Is Insufficient for General Downstream Tasks!} Aligning positive pairs only is not enough for general downstream tasks, as it can lead to a degenerate solution of representations. Even though the positive target is required to satisfy uniform marginal constraints, a transport plan can place all its mass on positive pairs while satisfying both marginals. In the low entropy limit ($\epsilon \downarrow 0$), no negative mass is forced at all. Entropic regularization creates some indirect spillover, but it neither labels which pairs are negatives nor requires a large anchor-negative separation. Consequently, many very different cross-class geometries can yield essentially the same positive-matching objective. For example, consider a setting with ten classes in which the encoder collapses all samples within each class to a single point. Suppose that the representations of eight classes are clustered very close together, while the remaining two classes are located far away. The OT objective can still correctly match all positive pairs according to the class labels, even though the resulting representation geometry is clearly undesirable: most classes remain poorly separated, which can impair subsequent training and downstream discrimination. Thus, correct positive matching alone does not guarantee effective negative repulsion or a well-structured class geometry. Prior work also does not support the claim that mass-conservation constraints alone induce sufficiently strong negative repulsion. In fact, \cite{shi2023understanding} suggests the opposite: the marginal constraints must be relaxed to recover an InfoNCE-like objective.

%%%%%%%%%%%%%%%%%%%%%%%%%%%%%%%%%%%%%%%%%%%%%%%%%%%%%%%%%%%%
\paragraph{Differences from closely related methods}
\label{app:further_discussion}
\begin{enumerate}
    \item We acknowledge that \cite{shi2023understanding, shi2024ot, piran2024contrasting} propose several variants of OT-based contrastive losses. However, these methods largely remain within the standard OT framework, mainly modifying the regularization or the positive-matching objective to strengthen positive alignment. In contrast, our method takes a different route: it introduces an additional mechanism that explicitly pushes negative pairs apart, adding a complementary “negative-separation” dimension beyond the vanilla OT formulation. Importantly, starting from the same vanilla OT baseline, our approach can also be extended to incorporate their variants. Therefore, for clarity and to isolate the core effect of our contribution, in the motivation section and the toy synthetic visualizations we compare our method only against standard OT.
    \item Non-contrastive Learning: The core idea of this work is to incorporate a negative repulsion mechanism into an Optimal Transport (OT)-based method to reinforce contrastiveness. Therefore, it is essential to compare this approach with other non-contrastive methods. Some existing methods focus solely on aligning positive pairs during pretraining and still achieve good performance on downstream tasks, e.g., \cite{grill2020bootstrap, bardes2021vicreg}. However, our approach differs from them in two key aspects:
    \begin{itemize}
        \item Multimodal Task Motivation: BYOL and VICReg were originally developed for self-supervised representation learning. However, to the best of our knowledge, there has been no work that integrates BYOL or VICReg with CLIP; meanwhile, our framework is applicable.
        \item Meaning of ``collapse'': It is well known that BYOL and VICReg are designed to avoid collapse. Meanwhile, our methods are designed to achieve a beneficial collapse configuration. However, it should be noted that these two kinds of collapse are different. In BYOL/VICReg, “avoid collapse” means avoiding the trivial constant-solution representation. Meanwhile, our methods aim for Neural Collapse is different: samples from the same class collapse to their class mean, but different classes remain maximally separated in an ETF-like structure. That is a class-structured collapse, not the trivial “everything maps to the same vector” collapse feared in self-supervised learning.
    \end{itemize}
\end{enumerate}

\newpage
\section{Pseudo code for supervised learning setting} 
\label{app:pseudo_for_specific_setting}
% \vspace{-0.5cm}
\begin{algorithm}[htp!] 
\scriptsize
\caption{Supervised CL with Neg-MMIOT-CL}
\label{algo: otmm}
\begin{algorithmic}
%\footnotesize
    \Require Dataset $\mathcal{D}$; Batch sampler $\mathcal{B}$; shaping function $\psi(\cdot)$; encoder architecture $f$; learning rate $\eta$; max\_epochs $T$; regularization parameter $\varepsilon$; number of Sinkhorn iterations $n_{it}$.
    \State Initialize encoder parameters $\theta$
    \For{$t = 1, \ldots, T$}
        \State Step 0: Sample a mini-batch $\{(x_i,y_i)\}_{i=1}^{B}$
        \State Step 1: Compute embeddings $z_i \gets f^\theta(x_i)$; $\;$normalize $z_i \leftarrow z_i/\|z_i\|$
        \State Step 2: Build admissible index set $\mathcal{A} = \{(i,j,k) \| i,j,k \text{ are distinct }; y_i = y_j; y_i \neq y_k\}$
        \State \quad \quad \quad Construct ground-truth coupling $\tilde{P}\in\{0,1\}^{B\times B\times B}$ by
        \[
        \tilde{P}_{ijk}=\begin{cases}
        1, & (i,j,k)\in \mathcal{A} \\
        0, & \text{otherwise}
        \end{cases}
        \]
        \State \quad \quad \quad Normalize $\tilde{P} \leftarrow \frac{\tilde{P}}{\tilde{P}.sum()}$
        \State Step 4: Form cost tensor $C^\theta\in\mathbb{R}^{B\times B\times B}$ with
        \[
        C^\theta_{ijk} \gets \psi\!\Big( \frac{\langle z_i, z_j \rangle-\langle z_i, z_k \rangle}{\tau} \Big) \text{ if i,j,k are different}
        \]
        \[
        C^\theta_{ijk} \gets \infty \text{ if i=j or j = k or  k = i}
        \]
        \State Step 5: Solve Entropic Multi-Marginal OT (\ref{eq:clot_final}):
        \[
        P^{C^\theta} \;\gets\; \textsc{MM\text{-}Sinkhorn}(C^\theta,\varepsilon, n_{it})
        \]
        \State Step 6: Compute objective $\mathcal{L}(\theta)\gets \mathrm{KL}\!\big(\tilde{P}\,\|\,P^{C^\theta}\big)$
        \State Step 7: Backpropagate $\nabla_\theta \mathcal{L}$ and update $\theta$ by AdamW
    \EndFor
    \State \textbf{Output:} trained encoder $f^\theta$
\end{algorithmic}
\end{algorithm}

\begin{algorithm}[htp!] 
\scriptsize
\caption{Supervised CL with Neg-IOT-CL-PushPull}
\label{algo: iotpushpull}
\begin{algorithmic}
    \Require Dataset $\mathcal{D}$; Batch sampler $\mathcal{B}$; shaping function $\psi(\cdot)$; encoder architecture $f$; learning rate $\eta$; max\_epochs $T$; regularization parameters $\varepsilon_{+}, \varepsilon_{-}$.
    \State Initialize encoder parameters $\theta$
    \For{$t = 1, \ldots, T$}
        \State Step 0: Sample mini-batch $\{(x_i,y_i)\}_{i=1}^{B}$
        \State Step 1: Compute embeddings $z_i \gets f^\theta(x_i)$, then normalize $z_i \leftarrow z_i/\|z_i\|$
        \State Step 2: Build positive and negative index set $\mathcal{A}^+ = \{(i,j) \| i \neq j; y_i = y_j\} ; \mathcal{A}^+ = \{(i,k) \| i \neq k; y_i \neq y_k\}$
        \State \quad \quad \quad Construct ground-truth coupling
        \[
        (\tilde P_{+})_{ij}=\begin{cases}
        1, & (i,j)\in \mathcal{A}^+ \\
        0, & \text{otherwise}
        \end{cases}
        \qquad
        (\tilde P_{-})_{ik} \!=\! =\begin{cases}
        1, & (i,k)\in \mathcal{A}^- \\
        0, & \text{otherwise}
        \end{cases}
        \]
        \State \quad \quad \quad Normalize $\tilde P_{+} \leftarrow \frac{\tilde P_{+}}{\tilde P_{+}.sum()}$; $\tilde P_{-} \leftarrow \frac{\tilde P_{-}}{\tilde P_{-}.sum()}$
        \State Step 3: Build cost matrix $C^\theta \in \mathbb{R}^{B\times B}$ with 
        $$C^\theta_{ij} \gets \psi\!\big(\frac{\langle z_i, z_j \rangle}{\tau}\big) \text{ if } i \neq j$$
        $$C^\theta_{ij} \gets \infty \text{ if } i = j$$
        \State Step 4: Solve \textbf{Positive OT:}
        \Statex \hspace{1.5em}$P^{C^\theta}_{+} \gets \textsc{Sinkhorn}\!\left(C^\theta;\ \varepsilon_{+},\ n_{it}\right)$
        \State Step 5: Solve \textbf{Negative OT:}
        \Statex \hspace{1.5em}$P^{C^\theta}_{-} \gets \textsc{Sinkhorn}\!\left((-C^\theta);\ \varepsilon_{-},\ n_{it}\right)$
        \State Step 6: Compute loss \(
        \mathcal L(\theta)= \mathrm{KL}(\tilde P_{+}\Vert P^{C^\theta}_{+}) + \mathrm{KL}(\tilde P_{-}\Vert P^{C^\theta}_{-})
        \)
        \State Step 7: Backpropagate $\nabla_\theta \mathcal{L}$ and update $\theta$ by AdamW
    \EndFor
    \State \textbf{Output:} Encoder $f^\theta$
\end{algorithmic}
\end{algorithm}

\begin{algorithm}[htp!]
% \label{alg:4}
\small
\caption{Class-Uniform Batch Sampler (only for Synthetic Experiments under the Theoretical Setting)}
\label{alg:4}
\begin{algorithmic}
    \Require Dataset $\mathcal{D} = \{(x_i,y_i)\}_{i=1}^{Kn}$; Number of classes $K$; total samples per class $n$; batch size per class $b$; carryover ratio $\rho \in [0,1]$; number of epochs $T$.
    \State For each class $c \in \{1,\dots,K\}$, let $\mathcal{I}_{\mathfrak{c}}$ be the index set of all samples with label $c$.
    \State Initialize previous per-class batch indices $\mathcal{S}_{\mathfrak{c}}^{(0)} = \emptyset$ for all $c$.
    \State Initialize remaining unsampled indices $\mathcal{R}_{\mathfrak{c}}^{(0)} = \mathcal{I}_{\mathfrak{c}}$ for all $c$.
    \For{epoch $t = 1,2,\dots, T$}
        \For{$c = 1, \ldots, K$}
            \If{$t = 1$ \textbf{or} $\mathcal{S}_{\mathfrak{c}}^{(t-1)} = \emptyset$}
                \If{$|\mathcal{R}_{\mathfrak{c}}^{(t-1)}| < b$}
                    \State \textbf{Reset coverage cycle:} $\mathcal{R}_{\mathfrak{c}}^{(t-1)} \gets \mathcal{I}_{\mathfrak{c}}$
                \EndIf
                \State Uniformly sample $b$ indices $\mathcal{S}_{\mathfrak{c}}^{(t)} \subset \mathcal{R}_{\mathfrak{c}}^{(t-1)}$ without replacement.
                \State Update remaining set $\mathcal{R}_{\mathfrak{c}}^{(t)} \gets \mathcal{R}_{\mathfrak{c}}^{(t-1)} \setminus \mathcal{S}_{\mathfrak{c}}^{(t)}$.
            \Else
                \State $m \gets \lfloor \rho b \rfloor$ \Comment{number of samples to keep from last mini-batch}
                \State Uniformly sample $m$ indices $\mathcal{K}_{\mathfrak{c}}^{(t)} \subset \mathcal{S}_{\mathfrak{c}}^{(t-1)}$ without replacement.
                \State $\ell \gets b - m$ \Comment{number of new samples needed}
                \State $\mathcal{A}_{\mathfrak{c}}^{(t)} \gets \mathcal{R}_{\mathfrak{c}}^{(t-1)} \setminus \mathcal{K}_{\mathfrak{c}}^{(t)}$ \Comment{fresh candidates not kept}
                \If{$|\mathcal{A}_{\mathfrak{c}}^{(t)}| \ge \ell$}
                    \State Uniformly sample $\ell$ indices $\mathcal{N}_{\mathfrak{c}}^{(t)} \subset \mathcal{A}_{\mathfrak{c}}^{(t)}$ without replacement.
                    \State $\mathcal{R}_{\mathfrak{c}}^{(t)} \gets \mathcal{R}_{\mathfrak{c}}^{(t-1)} \setminus \mathcal{N}_{\mathfrak{c}}^{(t)}$
                \Else
                    \Comment{not enough fresh points left: use all, then reset}
                    \State $\mathcal{N}_{\mathfrak{c},1}^{(t)} \gets \mathcal{A}_{\mathfrak{c}}^{(t)}$ \Comment{use all remaining fresh points}
                    \State $r \gets |\mathcal{N}_{\mathfrak{c},1}^{(t)}|$
                    \State $\ell' \gets \ell - r$ \Comment{additional samples needed}
                    \State \textbf{Reset coverage cycle:} $\mathcal{R}_{\mathfrak{c}}^{\text{new}} \gets \mathcal{I}_{\mathfrak{c}} \setminus \big(\mathcal{K}_{\mathfrak{c}}^{(t)} \cup \mathcal{N}_{\mathfrak{c},1}^{(t)}\big)$
                    \State Uniformly sample $\ell'$ indices $\mathcal{N}_{\mathfrak{c},2}^{(t)} \subset \mathcal{R}_{\mathfrak{c}}^{\text{new}}$ without replacement.
                    \State $\mathcal{N}_{\mathfrak{c}}^{(t)} \gets \mathcal{N}_{\mathfrak{c},1}^{(t)} \cup \mathcal{N}_{\mathfrak{c},2}^{(t)}$
                    \State $\mathcal{R}_{\mathfrak{c}}^{(t)} \gets \mathcal{R}_{\mathfrak{c}}^{\text{new}} \setminus \mathcal{N}_{\mathfrak{c},2}^{(t)}$
                \EndIf
                \State Set $\mathcal{S}_{\mathfrak{c}}^{(t)} \gets \mathcal{K}_{\mathfrak{c}}^{(t)} \cup \mathcal{N}_{\mathfrak{c}}^{(t)}$
            \EndIf
        \EndFor
        \State Construct mini-batch 
        \[
        \mathcal{B}^{(t)} = \{(x_i, y_i) : i \in \bigcup_{c=1}^{K} \mathcal{S}_{\mathfrak{c}}^{(t)}\}.
        \]
        \State \textbf{Output:} Mini-batch $\mathcal{B}^{(t)}$ to perform one training step at epoch $t$.
    \EndFor
\end{algorithmic}
\end{algorithm}

%%%%%%%%%%%%%%%%%%%%%%%%%%%%%%%%%%%%%%%%%%%%%%%%%%%%%%%%%%%%

\clearpage
\section{Proofs of theoretical results} \label{appendix:proof}
\subsection{Proof of Lemma \ref{lem:convexity_to_cost}}
\label{apd: proof for lem:convexity_to_cost}
\begin{proof}
First, under Assumption~A1, $\mathrm{support}(\tilde{P}) \subseteq [N]^3_{\neq}$. Since for all $C\in \mathcal{C}$, $C_{ijk} = +\infty$ for all $(i,j,k) \notin [N]^{3}_{\neq}$, if $\langle C, P\rangle$ is to be finite, we must have $\mathrm{support}(P) \subseteq [N]^3_{\neq}$.
Next, all three marginal PMFs of the uniform PMF over $[N]^3_{\neq}$ 
%i.e., $$\frac{1}{N(N-1)(N-2}\mathbf{1}((i,j,k) \in \mathcal{S}_N),$$ 
are equal to $\mathbf{u}_N$. Thus, the set of all PMFs in $\Pi(\mathbf{u}_N,\mathbf{u}_N,\mathbf{u}_N)$ with support within $[N]^3_{\neq}$ is not empty, i.e., $\Delta([N]^3_{\neq}) \cap \Pi(\mathbf{u}_N,\mathbf{u}_N,\mathbf{u}_N) \neq \{\}$.

For all $P$ with support within $[N]^3_{\neq}$, if $f(P) := \big[\langle C, P \rangle - \varepsilon H(P)\big]$, then $f(\cdot)$ is a strictly convex differentiable function of $P$ since $\langle C, P \rangle$ is a linear function of $P$ and $-H(P)$ is a strictly convex differentiable function of $P$. Thus the objective function $f(P)$ is strictly convex and differentiable and the constraint set  $\Delta([N]^3_{\neq}) \cap \Pi(\mathbf{u}_N,\mathbf{u}_N,\mathbf{u}_N)$ is not empty and defined via linear equality constraints of the marginals and the support (and is therefore a non-empty convex set). Therefore, from basic results in Convex Optimization Theory, e.g., see Proposition~5.3.3 in \cite{bertsekas2009convex}, there exists a unique minimizer $P^C$ to the problem $\argmin_{P\in \Pi(\mathbf{u}_N,\mathbf{u}_N,\mathbf{u}_N)} f(P)$ defined by the stationary point of the Lagrangian function with respect to $P$. Let
{
\begin{align*}
\mathcal{L}(P,\lambda,\mu,\nu) := 
f(P) 
&+ \sum_{i'\in[N]} \!\!\!\lambda_{i'}\Big(\langle\mathbf{1}(i=i'),P\rangle - \tfrac{1}{N}\Big)
+ \sum_{j'\in[N]} \!\!\!\mu_{j'}\Big(\langle\mathbf{1}(j=j'),P\rangle - \tfrac{1}{N} \Big) \\
&{}+ \sum_{k'\in[N]} \!\!\!\nu_{k'}\Big(\langle\mathbf{1}(k=k'),P\rangle - \tfrac{1}{N} \Big)
\end{align*}}denote the Lagrangian function with $3N$ Lagrange multipliers $\{\lambda_{i'}, \mu_{j'}, \nu_{k'}: i',j',k' \in [N] \}$. Then there is no duality gap and there exists a choice of $3N$ Lagrange multipliers such that
\[
\forall (i, j, k) \in [N]^3_{\neq}, \quad  \left.\frac{\partial \mathcal{L}}{\partial P_{ijk}}\right|_{P = P^C} = 0 \Rightarrow C_{ijk}+\varepsilon\log P^C_{ijk} + \lambda_i + \mu_j + \nu_k = 0.
\]

Multiplying the stationarity condition for the triplet $(i,j,k)$ by $(\tilde{P}_{ijk}-P^C_{ijk})$ and summing over all $(i, j, k) \in [N]^3_{\neq}$ we get:
$$
\sum_{(i,j,k) \in [N]^3_{\neq}}\left(C_{i j k}+\varepsilon \log P^C_{ijk} + \lambda_i + \mu_j + \nu_k\right)\left(\tilde{P}_{ijk}-P^C_{ijk}\right)=0 .
$$
Since all three marginals of $\tilde{P}$ and $P^C$ are uniform, i.e., $\mathbf{u}_N$, the summations over all the dual terms vanish, i.e.,
{
$$
\sum_{(i,j,k) \in [N]^3_{\neq}} \!\!\!\!\!\!\!\!\lambda_i\left(\tilde{P}_{ijk}-P^C_{ijk}\right) =
\!\!\!\!\!\!\!\sum_{(i,j,k) \in [N]^3} 
\!\!\!\!\!\!\!\lambda_i\left(\tilde{P}_{ijk}-P^C_{ijk}\right) =
$$$$
\sum_{i\in [N]} 
\!\!\lambda_i\Big(\!\!\sum_{j,k \in [N]} \!\!\! \tilde{P}_{ijk} -\!\!\sum_{j,k\in [N]} \!\!\!P^C_{ijk}\Big) =
\sum_{i\in [N]} 
\!\lambda_i\left(\tfrac{1}{N} - \tfrac{1}{N}\right) = 0,
%\sum_{i,j,k \in [N]} \mu_j\left(\tilde{P}_{ijk}-P^C_{ijk}\right) = 
%\sum_{i,j,k \in [N]} \nu_k\left(\tilde{P}_{ijk}-P^C_{ijk}\right) = 
%0,
$$
}
and similarly,
$$
\sum_{(i,j,k) \in [N]^3_{\neq}} \mu_j\left(\tilde{P}_{ijk}-P^C_{ijk}\right) = 
\sum_{(i,j,k) \in [N]^3_{\neq}} \nu_k\left(\tilde{P}_{ijk}-P^C_{ijk}\right) = 
0.
$$
Therefore,
$$
\left\langle C, \tilde{P} - P^C\right\rangle+\varepsilon \sum_{(i,j,k) \in [N]^3_{\neq}}\left(\tilde{P}_{ijk}-P^C_{ijk}\right) \log P^C_{ijk}=0,
$$
which implies that
$$
-\varepsilon \sum_{(i,j,k)\in[N]^{3}_{\neq}} \tilde{P}_{ijk} \log P^C_{ijk}=\left\langle C, \tilde{P} - P^C\right\rangle-\varepsilon \sum_{(i,j,k) \in [N]^{3}_{\neq}} P^C_{ijk} \log P^C_{ijk} .
$$
Hence, for all $C \in \mathcal{C}$,
\begin{align*}
\varepsilon ~g(C) = 
\varepsilon ~\mathrm{KL}\left(\tilde{P} \| P^C\right) & =\varepsilon \!\!\!\!\!\sum_{(i, j , k) \in \mathrm{support}(\tilde{P})} \!\!\!\!\!\!\!\!\!\tilde{P}_{ijk} \log \left(\frac{\tilde{P}_{ijk}}{P^C_{ijk}}\right) \\
& =\varepsilon \!\!\!\!\!\sum_{(i, j , k) \in \mathrm{support}(\tilde{P})} \!\!\!\!\!\!\!\!\!\tilde{P}_{ijk} \log \tilde{P}_{ijk}-\varepsilon \!\!\!\!\!\sum_{(i, j , k) \in \mathrm{support}(\tilde{P})} \!\!\!\!\!\!\!\!\!\tilde{P}_{ijk} \log P^{C}_{ijk} \\
& =\varepsilon \!\!\!\!\!\sum_{(i, j , k) \in \mathrm{support}(\tilde{P})} \!\!\!\!\!\!\!\!\!\tilde{P}_{ijk} \log \tilde{P}_{ijk}-\varepsilon \!\!\!\!\!\sum_{i, j , k \in [N]^3_{\neq}} 
\!\!\!\!\!\!\tilde{P}_{ijk} \log P^{C}_{ijk} \\
& =\varepsilon \!\!\!\!\!\sum_{(i, j , k) \in \mathrm{support}(\tilde{P})} \!\!\!\!\!\!\!\!\!\tilde{P}_{ijk} \log \tilde{P}_{ijk}+\left\langle C, \tilde{P} - P^{C}\right\rangle-\varepsilon \!\!\!\!\!\!\!\sum_{(i, j , k) \in [N]^3_{\neq}} 
\!\!\!\!\!\!\!\!\!P^{C}_{ijk} \log P^{C}_{ijk} \\
& = [\langle C, \tilde{P}\rangle-\varepsilon H(\tilde{P})] - [\langle C, P^C \rangle-\varepsilon H(P^C)] \\
& = [\langle C, \tilde{P}\rangle-\varepsilon H(\tilde{P})] - \min_{P \in \Pi(\mathbf{u}_N,\mathbf{u}_N,\mathbf{u}_N) \cap \Delta([N]^3_{\neq})}
  \Big[\langle C,P\rangle-\varepsilon\,H(P)\Big].
\end{align*}
The first term on the right side of the last equality is an affine (and therefore convex) function of $C$. The pointwise minimum of a family of affine functions of $C$ is a concave function of $C$ \cite{boyd2004convex}. Since $\varepsilon > 0$, it follows that $g(C) = \mathrm{KL}(\tilde{P}||P^C)$ is a convex function of $C$ and this concludes the proof.
\end{proof}

We note that the stationarity of the Lagrangian implies that $P^C$ has the following from:
\[
\forall (i, j, k) \in [N]^3_{\neq}, \quad P^C_{ijk} = e^{-\frac{C_{ijk} + \lambda_i + \mu_j + \nu_k}{\varepsilon}}
\]
where the Lagrange multipliers $\lambda,\mu,\nu$ can be chosen to satisfy all three uniform marginal constraints. In general, $\lambda,\mu,\nu$ will depend on $C$.

\subsection{Proof of Lemma \ref{lem:achieve_sym}}
\label{apd: proof for lem:achieve_sym}
\begin{proof}
We first prove that uniform marginals are preserved under any $\pi \in \Gamma$. Indeed, take any $P \in \Pi(\mathbf{u}_N,\mathbf{u}_N,\mathbf{u}_N)$. Then for the first marginal of $P \circ \pi$ is given by
$$
\sum_{j, k \in [N]}(P \circ \pi)_{i j k}=\sum_{j, k  \in [N]} P_{\pi(i) \pi(j) \pi(k)} .
$$
Let $i^{\prime}=\pi(i), j^{\prime}=\pi(j), k^{\prime}=\pi(k)$. Because $\pi$ is a permutation, the mapping $(j, k) \mapsto\left(j^{\prime}, k^{\prime}\right)$ is bijective. Therefore,
$$
\sum_{j, k \in [N]}(P\circ \pi)_{i j k}=\sum_{j^{\prime}, k^{\prime} \in [N]} P_{i^{\prime} j^{\prime} k^{\prime}}=\frac{1}{N}.
$$
Similarly, 
$$
\sum_{i, k \in [N]}(P \circ \pi)_{i j k}=\frac{1}{N}, \quad \sum_{i, j \in [N]}(P \circ \pi)_{i j k}=\frac{1}{N} .
$$
Thus for all $\pi \in \Gamma$, $P \in \Pi(\mathbf{u}_N,\mathbf{u}_N,\mathbf{u}_N) \Leftrightarrow P \circ \pi \in \Pi(\mathbf{u}_N,\mathbf{u}_N,\mathbf{u}_N)$.
We next prove that the admissible set $\mathcal{A}$ is also invariant to all $\pi \in \Gamma$ so that $\tilde{P} \circ \pi = \tilde{P}$.
Recall that a triplet $(i, j, k)$ is in $\mathcal{A}$ iff $y_j=y_i, j \neq i, y_k \neq y_i$. 
Since $\pi$ only permutes samples within classes or permutes whole classes, it preserves equality of labels, inequality of labels, and distinctness of indices. Therefore,
$$
(i, j, k) \in \mathcal{A} \Longleftrightarrow(\pi(i), \pi(j), \pi(k)) \in \mathcal{A} .
$$
Since 
$\tilde{P}_{i j k}=\frac{1}{|\mathcal{A}|} \mathbf{1}((i, j, k) \in \mathcal{A})$,
$$
(\tilde{P} \circ \pi)_{i j k}=\tilde{P}_{\pi(i) \pi(j) \pi(k)}=\frac{1}{|\mathcal{A}|} \mathbf{1}\left((\pi(i), \pi(j), \pi(k)) \in \mathcal{A}\right)=\frac{1}{|\mathcal{A}|} \mathbf{1}((i, j, k) \in \mathcal{A})=\tilde{P}_{i j k} .
$$
Hence,
$$
\tilde{P} \circ \pi =\tilde{P} .
$$
Now consider the relabeled cost $C \circ \pi$ with corresponding entropic-OT minimizer $P^{C\circ \pi}$, i.e.,
$$
P^{C \circ \pi} = \argmin_{P\in\Pi} \Big[\langle C \circ \pi, P\rangle -\varepsilon\,H(P)\Big].
$$
Since permutations only reorder the entries, we have 
$$
\langle C \circ \pi, P\rangle=\left\langle C, P \circ \pi^{-1}\right\rangle
$$
and 
$$
H(P)=H(P\circ \pi^{-1}),
$$
since entropy is invariant to any permutation of the probability masses. Therefore,
$$
P^{C \circ \pi} = \argmin_{P\in\Pi} \Big[\langle C \circ \pi, P\rangle -\varepsilon\,H(P)\Big] = \argmin_{P\in\Pi} \Big[\langle C, P \circ \pi^{-1}\rangle -\varepsilon\,H(P\circ \pi^{-1})\Big]
$$
$$
\Rightarrow P^{C \circ \pi} \circ \pi^{-1} = 
\argmin_{P' \in\Pi} \Big[\langle C, P' \rangle -\varepsilon\,H(P')\Big] = P^{C}.
$$
This proves that for all $\pi \in \Gamma$,
$$
P^{C \circ \pi} = P^{C}\circ \pi.
$$
Similarly to entropy, the KL divergence is also invariant under identical permutations applied to both its PMF arguments. Therefore, for all $\pi \in \Gamma$,  
$$g(C\circ \pi) = \mathrm{KL}(\tilde{P},P^{C\circ \pi}) = \mathrm{KL}(\tilde{P},P^{C}\circ \pi) = \mathrm{KL}(\tilde{P} \circ \pi,P^{C}\circ \pi) = \mathrm{KL}(\tilde{P},P^{C}) =  
g(C).$$
Finally, by the convexity of $g(\cdot)$ proved in Lemma~\ref{lem:convexity_to_cost} and Jensen's inequality, we have $g(\bar{C}) = g(|\Gamma|^{-1} \sum_{\pi \in \Gamma} C\circ \pi) \leq |\Gamma|^{-1} \sum_{\pi \in \Gamma} g(C\circ \pi) = g(C)$. From the very definition of $\bar{C}$, for all $\pi \in \Gamma$ and all $i,j,k\in[N]$, $(\bar{C}\circ \pi)_{ijk} = \bar{C}_{\pi(i)\pi(j)\pi(k)} = |\Gamma|^{-1} \sum_{\pi' \in \Gamma} (C\circ \pi')_{\pi(i)\pi(j)\pi(k)} = |\Gamma|^{-1} \sum_{\pi' \in \Gamma} C_{\pi'(\pi(i))\pi'(\pi(j))\pi'(\pi(k))} = |\Gamma|^{-1} \sum_{\pi'' \in \Gamma} C_{\pi''(i)\pi''(j)\pi''(k)} = |\Gamma|^{-1} \sum_{\pi'' \in \Gamma} (C\circ \pi'')_{ijk} = \bar{C}_{ijk}$. Therefore, $\bar{C}\circ \pi = \bar{C}$ completing the proof.
\end{proof}

\subsection{Proof of Lemma \ref{lem:achieve_homo}}
\label{apd: proof of lem:achieve_homo}
\begin{proof} 
 First, let us note that every $C \in \tilde{C}$ depends on the sample representations $z_1, \ldots, z_N$ only through their $N \times N$  gram matrix $A$, where for all $i, j = 1, \ldots, N$,  $A_{ij} := z_i^\top z_j$. Since any Gram matrix is symmetric and positive semi-definite, we have   $A=A^\top, A\succeq 0$. Since all representations have unit norm, for all $i = 1, \ldots, N$, $A_{ii} = 1$, i.e.,  $\operatorname{diag}(A)=\mathbf{1}$. Since the representations live in $\mathbb{R}^d$, the rank of $A$ is not more than $d$. In the following, we will first relax the rank constraint on the gram matrix $A$ and characterize the optimal solution. The optimal solution without rank constraints will turn out to belong to the two-distance class-homogeneous family. This characterization will be utilized to prove Theorem~\ref{thm:nc-neg-mmot} where the optimal solution will be shown to have rank $K-1$, which does not exceed $d$.
    
Let $\mathcal E_N := \left\{ A\in {\mathbb R}^{N\times N}: A = A^\top, A\succeq 0,\  \operatorname{diag}(A)=\mathbf1\right\}$ which is a convex set (being the intersection of the convex set of all positive semi-definite matrices and the convex set of all symmetric matrices having a unit diagonal) and is referred to as an elliptope in convex optimization. Note that every $A\in\mathcal E_N$ admits a factorization $A=Z^\top Z$, where the columns of Z are unit-norm vectors in ${\mathbb R}^{\operatorname{rank}(A)}$. Thus, $\mathcal E_N$ is precisely the rank-unconstrained Gram-matrix feasible set. It is a non-empty convex set. Now define $\tilde{\mathcal{C}} = \left\{ C \in \mathbb{R}^{N\times N\times N} \,:\, C_{ijk} = \psi\!\left(A_{ij} - A_{ik}\right), A \in E_N \right\}$.
    
    For an affine function $\psi(t)=at+b$, the finite entries of the cost tensor satisfy $C_{ijk}(A) = \frac{a}{\tau} (A_{ij} - A_{ik}) + b$. Hence, for any $A_1, A_2 \in \mathcal{E}_N$ and any $\lambda \in [0,1]$, we have $(1-\lambda) C(A_1)+ \lambda C(A_2) = C(A_\lambda)$ where $A_\lambda := (1-\lambda) A_1 + \lambda A_2$ belongs to  $\mathcal{E}_N$ since $\mathcal{E}_N$ is a convex set. Therefore, $\tilde{C}$ is a convex set. 
    
    For any $C\in \tilde{\mathcal{C}}$ and any $\pi\in \Gamma$, the permuted tensor $C\circ \pi$ also belongs to $\tilde{\mathcal{C}}$. Hence, when $\psi$ is affine, by the convexity of $\tilde{C}$, $\bar{C} := |\Gamma|^{-1} \sum_{\pi \in \Gamma} C\circ \pi \in \tilde{\mathcal{C}}$. 
    Thus by Lemma~\ref{lem:achieve_sym} there is an optimal solution, denoted by $C^{\bar{\theta}}$, which is fully permutation-invariant (or $\Gamma$-invariant).
    Let $C^{\bar{\theta}}$ correspond toa Gram matrix $\bar{A}$ of embedding vectors $\{z_i\}_{i=1}^N$. Note that the dimension $d$ of the $z_i$'s is equal to or greater than $N$. We will prove that there is a choice of $z_i$'s  that satisfies the two-distance, class-homogeneous property. This will be used later to prove the result of Theorem~\ref{thm:nc-neg-mmot} which would reveal that the dimension of optimal embeddings can be tightened to $d \geq K-1$.    
    
    Let $\mathcal{I}_{\mathfrak{a}}=\left\{i: y_i=\mathfrak{a}\right\}$ be class $\mathfrak{a}$. Fix an anchor $i \in \mathcal{I}_{\mathfrak{a}}$. We have these following comments:
\begin{itemize}
    \item  All within-class similarities from anchor $i$ are equal. Indeed; take $j, k \in \mathcal{I}_{\mathfrak{a}} \backslash\{i\}, j \neq k$. Swapping $j$ and $k$ is a permutation in $\Gamma$, since $C^{\bar{\theta}}$ is $\mathrm{\Gamma}$-invariant, 
    $$
        C^{\bar{\theta}}_{ijk} = C^{\bar{\theta}}_{ikj}
    $$
    or
    $$
        \psi\left(\frac{z_i^{\top} z_j-z_i^{\top} z_k}{\tau}\right) = \psi\left(\frac{z_i^{\top} z_k-z_i^{\top} z_j}{\tau}\right).
    $$
    Because $\psi$ is strictly monotone, it is injective. Hence $z_i^{\top} z_j-z_i^{\top} z_k=z_i^{\top} z_k-z_i^{\top} z_j$, so $z_i^{\top} z_j=z_i^{\top} z_k$. Thus for this anchor $i$, every within-class off-diagonal similarity is the same. Call it $\alpha_i$. 
    \item For any fixed other class $\mathfrak{b} \neq \mathfrak{a}$, all similarities from anchor $i$ to class $\mathfrak{b}$ are equal. Indeed; take $j, k \in I_{\mathfrak{b}}, j \neq k$. Again swapping $j, k$ lies in $\Gamma$, and
    $$
        C^{\bar{\theta}}_{ijk} = C^{\bar{\theta}}_{ikj}
    $$
    or
    $$
        \psi\left(\frac{z_i^{\top} z_j-z_i^{\top} z_k}{\tau}\right) = \psi\left(\frac{z_i^{\top} z_k-z_i^{\top} z_j}{\tau}\right)
    $$
    Injectivity of $\psi$ gives $z_i^{\top} z_j=z_i^{\top} z_k$. Hence for each foreign class $\mathfrak{c} \neq \mathfrak{a}$, anchor $i$ sees all samples in $I_\mathfrak{c}$ with one common similarity. Call it $\beta_{i, \mathfrak{c}} $.
    \item Take $j \in \mathcal{I}_{\mathfrak{a}} \backslash\{i\}, k \in I_{\mathfrak{c}}$ and $\ell \in I_{\mathfrak{d}}$, where $\mathfrak{c}, \mathfrak{d} \neq \mathfrak{a}$. A permutation in $\Gamma$ that exchanges class $\mathfrak{c}$ and class $\mathfrak{d}$ while fixing class $\mathfrak{a}$ sends $(i, j, k)$ to $(i, j, \ell)$. Therefore
    $$
    C^{\bar{\theta}}_{i j k}=C^{\bar{\theta}}_{i j \ell}
    $$
    Then
    $$
    \psi\left(\frac{\alpha_i-\beta_{i, \mathfrak{c}} }{\tau}\right)=\psi\left(\frac{\alpha_i-\beta_{i, \mathfrak{d}} }{\tau}\right) .
    $$
    Injectivity of $\psi$ implies $\beta_{i, \mathfrak{c}} =\beta_{i, \mathfrak{d}} $. So for anchor $i$, every cross-class similarity is the same. Call it $\beta_i$. Thus, for fixed anchor $i \in \mathcal{I}_{\mathfrak{a}}$,
    $$
    z_i^{\top} z_j= \begin{cases}\alpha_i, & y_j = \mathfrak{a}, j \neq i \\ \beta_i, & y_j \neq \mathfrak{a}\end{cases}
    $$
\end{itemize}
Now we show that $\alpha_i$ and $\beta_i$ are actually global constants. Indeed; first, if $i, i^{\prime} \in \mathcal{I}_{\mathfrak{a}}$, then
$$
\alpha_i=z_i^{\top} z_{i'}=z_{i'}^{\top} z_i=\alpha_{i^{\prime}}
$$
So $\alpha_i$ is constant within each class. Write that value as $\alpha^{(\mathfrak{a})}$. Next, take $i \in \mathcal{I}_{\mathfrak{a}}$ and $k \in I_{\mathfrak{b}}, a \neq b$. Then we have
$$
\beta_i = z_i^{\top} z_k= z_k^{\top} z_i=\beta_k,
$$
then all $\beta_i$ are equal. Call the common value $\beta$.
Finally, all admissible triples belong to one $\Gamma$-orbit. Hence $C^{\bar{\theta}}_{i j k}$ is the same for every $(i, j, k) \in \mathcal{A}$. For any admissible triple with anchor in class $\mathfrak{a}$,
$$
C^{\bar{\theta}}_{i j k}=\psi\left(\frac{\alpha^{(\mathfrak{a})}-\beta}{\tau}\right) .
$$
Since this is independent of the class $\mathfrak{a}$, injectivity of $\psi$ implies
$$
\alpha^{(\mathfrak{a})}=\alpha \quad \text { for all } \mathfrak{a} .
$$
Therefore
$$
z_i^{\top} z_j= \begin{cases}1, & i=j \\ \alpha, & i \neq j, y_i=y_j, \\ \beta, & y_i \neq y_j .\end{cases}
$$
So the minimizer lies in the two-distance class-homogeneous family. This completes the proof.
\end{proof}

\subsection{Proof of Theorem \ref{thm:nc-neg-mmot}}
\label{apd: proof for thm:nc-neg-mmot}
\begin{lemma}[NC/ETF for Neg-MMIOT with class-homogeneous regime of lemma \ref{lem:achieve_homo}]\label{lem:nc_for_homo}
Given the embeddings satisfy the \emph{class-homogeneous (two-distance) symmetry}:
there exist scalars $\alpha,\beta\in[-1,1]$ such that for all $i\neq j$,
\[
\langle z_i,z_j\rangle=
\begin{cases}
\alpha, & y_i=y_j,\\
\beta, & y_i\neq y_j.
\end{cases}
\tag{$\star$}
\]
Then, among all embeddings satisfying $(\star)$ and $\|z_i\|=1$, every global minimizer of Neg-MMIOT
must satisfy
\[
\alpha^\star=1,\qquad \beta^\star=-\frac{1}{K-1}.
\]
Consequently, within-class collapse holds (all samples in a class share the same representation),
and the $K$ class vectors $\{\mu_\mathfrak{c}\}_{\mathfrak{c}=1}^K$ form a \emph{simplex ETF}:
\[
\|\mu_\mathfrak{c}\|_2=1,\qquad \sum_{\mathfrak{c}=1}^K \mu_\mathfrak{c}=0,\qquad 
\langle \mu_\mathfrak{c},\mu_\mathfrak{c'}\rangle = -\frac{1}{K-1}\ \ \forall \mathfrak{c} \neq \mathfrak{c'}.
\]
\end{lemma}

\paragraph{Remark.} The simplex ETF conclusion requires the ambient dimension condition $d \geq K-1$. Indeed, the Gram matrix at ETF configuration is

$$
\mathcal{G}=\frac{K}{K-1} I_K-\frac{1}{K-1} \mathbf{1_K}{\mathbf{1_K}}^{\top}
$$
which has eigenvalues $K /(K-1)$ with multiplicity $K-1$ and 0 with multiplicity 1 . Therefore $\operatorname{rank}(\mathcal{G})=K-1$. Since $\mathcal{G}=\mu \mu^{\top}$ for $\mu=\left[\mu_1^{\top}, \ldots, \mu_K^{\top}\right]^{\top} \in \mathbb{R}^{K \times d}$, we must have $K-1= \operatorname{rank}(\mathcal{G}) \leq d$. Conversely, when $d \geq K-1$, such a configuration is realized by a regular simplex in $\mathbb{R}^{K-1}$, embedded into $\mathbb{R}^d$. Hence the ETF equality case is achievable if and only if $d \geq K-1$.

Moreover, Lemma \ref{lem:nc_for_homo} is the only step where the sign of monotonicity of $\psi$ matters. The proof uses the two-distance class-homogeneous ansatz, unit-norm embeddings, the balanced setting with uniform marginals, and $\varepsilon>0$. Under these assumptions, the inner OT solution depends only on the three cost levels $\psi(\Delta / \tau), \psi(-\Delta / \tau)$, and $\psi(0)$, where $\Delta=\alpha-\beta$. The strict decrease of $\psi$ is then used to show that the outer KL objective is strictly decreasing in $\Delta$. Mere injectivity is not enough for this monotonicity step. The bound $\beta \geq-1 /(K-1)$ uses only positive semidefiniteness of the class-representative Gram matrix, whereas the equality case $\beta=-1 /(K-1)$ requires the ambient-dimension condition $d \geq K-1$ in order to realize a regular simplex ETF. Thus $d \geq K-1$ is needed only for achievability of the ETF equality case, not for the monotonicity argument or the PSD lower bound. No convexity or differentiability of $\psi$, and no conditional-independence assumption, is used in this lemma.
\begin{proof}
Let $\Delta:=\alpha-\beta$. For distinct $i,j,k$:

\emph{(i)} If $(i,j,k)\in\mathcal{A}$ then $y_j=y_i$ and $y_k\neq y_i$, hence
$\langle z_i,z_j\rangle-\langle z_i,z_k\rangle=\alpha-\beta=\Delta$ and
\[
C^\theta_{ijk}=\psi(\Delta/\tau).
\]

\emph{(ii)} If $y_j\neq y_i$ and $y_k=y_i$,
then $\langle z_i,z_j\rangle-\langle z_i,z_k\rangle=\beta-\alpha=-\Delta$ and
\[
C^\theta_{ijk}=\psi(-\Delta/\tau).
\]

\emph{(iii)} Otherwise 
\[
C^\theta_{ijk}=\psi(0).
\]
Consider the inner objective. Form the Lagrangian with multipliers $(\lambda_i)_{i=1}^N$, $(\mu_j)_{j=1}^N$, $(\nu_k)_{k=1}^N$ for the three marginals constraints:
\begin{align*}
\mathcal{L}(P,\lambda,\mu, \nu)
= &
\sum_{i,j,k}\Big(C^\theta_{ijk}P_{ijk}+\varepsilon P_{ijk}(\log P_{ijk}-1)\Big)
+\sum_i \lambda_i\Big(\tfrac1N-\sum_{j,k}P_{ijk}\Big)
\\ & +\sum_j \mu_j\Big(\tfrac1N-\sum_{i,k}P_{ijk}\Big) \notag  +\sum_k \nu_k\Big(\tfrac1N-\sum_{i,j}P_{ijk}\Big).
\end{align*}
For any $(i,j,k)$, stationarity gives
\[
0=\frac{\partial \mathcal{L}}{\partial P_{ijk}}
= C^\theta_{ijk}+\varepsilon\log P_{ijk}-\lambda_i-\mu_j - \nu_k,
\]
hence
\[
P^{C^\theta}_{ijk}=\exp\!\Big(\frac{\lambda_i+\mu_j + \nu_k-C^\theta_{ijk}}{\varepsilon}\Big)
\quad \forall (i,j,k).
\tag{1}
\]
Because the constraints in feasible set are uniform in $i$, in $j$ and in $k$, meanwhile the cost tensor $C$ depends only on whether labels agree/disagree in the symmetric way above, the inner problem is invariant under the group $G$ of permutations that permute class labels, and permute samples within each class accordingly. By uniqueness of $P^{C^\theta}$, we must have $P^{C^\theta}$ invariant under every such permutation. Therefore all $\lambda_i$ must be equal (same orbit), and all $\mu_j$ must be equal, all $\nu_k$ must be equal. Plugging this into the KKT form gives:
$$
P^{C^\theta}_{i j k} \propto \exp \left(-C^\theta_{i j k} / \varepsilon\right) .
$$
and due to $\sum_{i,j,k} P^{C^\theta}_{ijk} = 1$
\[
P^{C^\theta}_{ijk}=\frac{e^{-C^\theta_{ijk}/\varepsilon}}{\sum_{(i,j,k)} e^{-C^\theta_{ijk}/\varepsilon}} = 
\begin{cases}
\frac{e^{-\psi(\Delta/\tau)/\varepsilon}}{\sum_{(i,j,k)} e^{-C^\theta_{ijk}/\varepsilon}}, & (i,j,k)\in\mathcal{A},\\[4pt]
\frac{e^{-\psi(-\Delta/\tau)/\varepsilon}}{\sum_{(i,j,k)} e^{-C^\theta_{ijk}/\varepsilon}}, & (i,j,k)\in\mathcal{A}^{\mathrm{op}} := \{(i,j,k): i,j,k \text{ are distinct }; y_i \neq y_j; y_i = y_k \}\\[4pt]
\frac{e^{-\psi(0)/\varepsilon}}{\sum_{(i,j,k)} e^{-C^\theta_{ijk}/\varepsilon}} 
& (i,j,k)\in\mathcal{R} := \text{ otherwise},
\end{cases},
\]
Now we compute the outer KL. Since $\widetilde P$ is uniform on $\mathcal{A}$,
\[
\mathrm{KL}(\widetilde P\|P^{C^\theta})
=\sum_{(i,j,k)\in\mathcal{A}}\frac{1}{|\mathcal{A}|}
\log\frac{1/|\mathcal{A}|}{P^{C^\theta}_{ijk}}
=
\log\Big(\frac{\sum_{(i,j,k)} e^{-C^\theta_{ijk}/\varepsilon}}{|\mathcal{A}|e^{-\psi(\Delta/\tau)/\varepsilon}}\Big).
\]
Let $m(\Delta):=\sum_{(i,j,k)\in\mathcal{A}}P^{C^\theta}_{ijk}
=|\mathcal{A}|\frac{e^{-\psi(\Delta/\tau)/\varepsilon}}{\sum_{(i,j,k)} e^{-C^\theta_{ijk}/\varepsilon}}.$
Then it becomes
\[
\mathrm{KL}(\widetilde P\|P^{C^\theta}) = -\log m(\Delta).
\]
Because $\psi$ is strictly decreasing, $\psi(\Delta/\tau)$ strictly decreases with $\Delta$,
so $e^{-\psi(\Delta/\tau)/\varepsilon}$ strictly increases.
Similarly $\psi(-\Delta/\tau)$ strictly increases with $\Delta$, so $e^{-\psi(-\Delta/\tau)/\varepsilon}$ strictly decreases. Then,
\[
m(\Delta)=\frac{|\mathcal{A}|e^{-\psi(\Delta/\tau)/\varepsilon}}{|\mathcal{A}|e^{-\psi(\Delta/\tau)/\varepsilon} + |\mathcal{A}^{\mathrm{op}}|e^{-\psi(-\Delta/\tau)/\varepsilon} + |\mathcal{R}|e^{-\psi(0)/\varepsilon}}
\]
is strictly increasing. Therefore,
$\mathrm{KL}(\widetilde P\|P^{C^\theta})$ is strictly decreasing in $\Delta$.
So the global minimizer of KL is exactly the maximizer of $\Delta=\alpha-\beta$.

We have $\alpha\le 1$ since $\langle z_i,z_j\rangle\le \|z_i\|\,\|z_j\|=1$.
For $\beta$, pick one representative $I_{\mathfrak{c}}$ from each class $\mathfrak{c}$.
Then the $K\times K$ Gram matrix $\mathcal{G}$ of $\{z_{I_{\mathfrak{c}}}\}_{\mathfrak{c}=1}^K$ has $1$ on the diagonal and $\beta$ off-diagonal:
\[
\mathcal{G}=(1-\beta)I_K + \beta\,\mathbf{1}\mathbf{1}^\top.
\]
This matrix must be PSD, so all eigenvalues are nonnegative. The eigenvalues are
$1-\beta$ (mult.\ $K-1$) and $1+(K-1)\beta$ (mult.\ $1$), hence $1+(K-1)\beta\ge 0$ and
\[
\beta \ge -\frac{1}{K-1}.
\]
Therefore
\[
\Delta=\alpha-\beta \le 1-\Big(-\frac{1}{K-1}\Big)=\frac{K}{K-1}.
\]
This upper bound is achievable by taking class vectors $\{\mu_\mathfrak{c}\}_{\mathfrak{c}=1}^K$ to be a regular simplex in $\mathbb{R}^{K-1}$
(i.e. $\langle \mu_\mathfrak{c},\mu_\mathfrak{c'}\rangle=-\frac{1}{K-1}$ for $\mathfrak{c} \neq \mathfrak{c'}$ and $\|\mu_\mathfrak{c}\|=1$),
and setting every sample in class $\mathfrak{c}$ equal to $\mu_\mathfrak{c}$.
Hence the global minimizer has
\[
\alpha^\star=1,\qquad \beta^\star=-\frac{1}{K-1}.
\]
Then for any $i\neq j$ with $y_i=y_j$,
$\langle z_i,z_j\rangle=1$ with $\|z_i\|=\|z_j\|=1$ implies $z_i=z_j$ (equality in Cauchy--Schwarz).
Thus all samples in class $\mathfrak{c}$ collapse to some unit vector $\mu_\mathfrak{c}$.

The between-class inner product is $\beta^\star=-\frac{1}{K-1}$, so for $\mathfrak{c} \neq \mathfrak{c'}$,
$\langle \mu_\mathfrak{c},\mu_\mathfrak{c'}\rangle=-\frac{1}{K-1}$.
Then
\[
\Big\|\sum_{\mathfrak{c}=1}^K \mu_\mathfrak{c}\Big\|^2
=\sum_{\mathfrak{c}=1}^K\|\mu_\mathfrak{c}\|^2 + \sum_{\mathfrak{c} \neq \mathfrak{c'}}\langle \mu_\mathfrak{c},\mu_\mathfrak{c'}\rangle
=K + K(K-1)\Big(-\frac{1}{K-1}\Big)=0,
\]
hence $\sum_{\mathfrak{c}=1}^K \mu_\mathfrak{c}=0$.
Therefore $\{\mu_\mathfrak{c}\}_{\mathfrak{c}=1}^K$ is a simplex ETF.
\end{proof}

\subsection{Proof of Proposition \ref{prop: stationary}}
We start with the proof for the Neg-MMIOT-CL method. From the proof in Appendix \ref{apd: proof for lem:convexity_to_cost}, we have 
$$
\varepsilon ~g(C) = [\langle C, \tilde{P}\rangle-\varepsilon H(\tilde{P})] - \min_{P \in \Pi(\mathbf{u}_N,\mathbf{u}_N,\mathbf{u}_N) \cap \Delta([N]^3_{\neq})}
  \Big[\langle C,P\rangle-\varepsilon\,H(P)\Big]
$$
Differentiating this identity with respect to $C$ using Danskin's Theorem, we get
$$
\nabla_C\,g(C)
=
\frac1\varepsilon(\widetilde P-P^C).$$
%\tncomment{Is that $P^{C^{\theta}}$ in the above equation, oh we changed to $P^{C}$ but somehow in proof of Theorem 3.4 we still use $P^{C^{\theta}}$.}
Thus, for the embedding-dependent cost $C(z)$
$$\frac{\partial g}{\partial C_{ijk}}
=
\frac1\varepsilon
\bigl(\widetilde P_{ijk}-P^C_{ijk}\bigr).$$
For distinct $i, j, k$,
$$
C_{i j k}(z)=\psi\left(\frac{z_i^{\top} z_j-z_i^{\top} z_k}{\tau}\right) .
$$
Let
$$
u_{i j k}(z)=\frac{z_i^{\top} z_j-z_i^{\top} z_k}{\tau} .
$$
Then
$$
C_{i j k}(z)=\psi\left(u_{i j k}(z)\right) .
$$
Differentiate $C_{i j k}$ with respect to a particular vector $z_r$.
If $r=i$, then
$$
\nabla_{z_r} C_{r j k}=\frac{1}{\tau} \psi^{\prime}\left(u_{r j k}\right)\left(z_j-z_k\right) .
$$
If $r=j$, then
$$
\nabla_{z_r} C_{i r k}=\frac{1}{\tau} \psi^{\prime}\left(u_{i r k}\right) z_i .
$$
If $r=k$, then
$$
\nabla_{z_r} C_{i j r}=-\frac{1}{\tau} \psi^{\prime}\left(u_{i j r}\right) z_i .
$$
And if $r \neq i,j,k$ then the gradient should be $0$. Therefore,

\begin{align*}
\nabla_{z_r} g(z) & =\sum_{j, k} \frac{1}{\varepsilon \tau}\left(\widetilde{P}_{r j k}-P_{r j k}^C\right) \psi^{\prime}\left(u_{r j k}\right)\left(z_j-z_k\right) \\
& +\sum_{i, k} \frac{1}{\varepsilon \tau}\left(\widetilde{P}_{i r k}-P_{i r k}^C\right) \psi^{\prime}\left(u_{i r k}\right) z_i \\
& -\sum_{i, j} \frac{1}{\varepsilon \tau}\left(\widetilde{P}_{i j r}-P_{i j r}^C\right) \psi^{\prime}\left(u_{i j r}\right) z_i .
\end{align*}

For distinct triples, define

$$
\rho_{i j k}(z)=\frac{1}{\varepsilon \tau}\left(\widetilde{P}_{i j k}-P_{i j k}^C\right) \psi^{\prime}\left(\frac{z_{i}^{\top} z_j-z_{i}^{\top} z_k}{\tau}\right),
$$

and set $\rho_{i j k}=0$ if $i, j, k$ are not all distinct.
Then

$$
\nabla_{z_r} g(z)=\sum_{j, k} \rho_{r j k}(z)\left(z_j-z_k\right)+\sum_{i, k} \rho_{i r k}(z) z_i-\sum_{i, j} \rho_{i j r}(z) z_i .
$$
Now the mean-dynamic flow is:
$$\frac{d}{dt} z_r
=
-
P_{z_r}^{\perp}\bigl(\nabla_{z_r}g(z)\bigr)
$$ where $\frac{d}{dt} z_r$ denotes the gradient descent of $z_r$ with respect to time/iteration $t$, and $P_{z_r}^{\perp}$ is the orthogonal projection at $z_r$.
We prove that NC/ETF is stationary as it makes $\frac{d}{dt} z_r = 0,  \forall r$. Indeed, 
fix $r$, and let $c=y_r$. At the NC/ETF configuration, every vector $z_i$ is one of the class vectors $\mu_1, \ldots, \mu_K$. Therefore the gradient $\nabla_{z_r} g(z)$ is a linear combination of class vectors:

$$
\nabla_{z_r} g(z)=a_c \mu_c+\sum_{d \neq c} a_d \mu_d .
$$

Since $\widetilde{P}$ is uniform on the admissible triples, at a highly symmetric ETF configuration, $C$ inherits the same label-permutation symmetry, and so do $P^C$, and $\rho_{i j k}(z)$. Therefore, all classes $d \neq c$ contribute symmetrically. Hence, the coefficient $a_d$ is the same for every $d \neq c$. Then

$$
\nabla_{z_r} g(z)=a_c \mu_c+a_d \sum_{d \neq c} \mu_d .
$$
Using $\sum_{d=1}^K \mu_d=0,$ we have $\sum_{d \neq c} \mu_d=-\mu_c$. Therefore
$$
\nabla_{z_r} g(z)=a_c \mu_c-a_d \mu_c=\left(a_c-a_d\right) \mu_c = \left(a_c-a_d\right) z_r.
$$
But $z_r=\mu_c$. Thus $\nabla_{z_r} g(z)$ is parallel to $z_r$. Hence its tangent projection is $P_{z_r}^{\perp}\left(\nabla_{z_r} g(z)\right) = 0$.

This completes the proof for Neg-MMIOT-CL. Following the same line of reasoning for Neg-IOT-CL-PushPull, we consider 

$$
g_{\mathrm{PP}} (C)=\mathrm{KL}\left(\widetilde{P}_{+} \| P_{+}^C\right)+\mathrm{KL}\left(\widetilde{P}_{-} \| P_{-}^C\right) .
$$

Danskin's theorem gives, on the allowed off-diagonal support,

$$
\nabla_C g_{\mathrm{PP}} (C)=\frac{\widetilde{P}_{+}-P_{+}^C}{\epsilon_{+}}+\frac{P_{-}^C-\widetilde{P}_{-}}{\epsilon_{-}}
$$

At a balanced NC/ETF configuration, the targets, costs, and both Sinkhorn plans are invariant to within-class permutations and class permutations. Therefore, for a sample in class $c$, its embedding gradient has the form

$$
a \mu_c+b \sum_{d \neq c} \mu_d=(a-b) \mu_c
$$

because $\sum_d \mu_d=0$. Its spherical tangent projection is consequently zero. Then we complete the full proof. Moreover, within the two-distance class-homogeneous family and for strictly decreasing $\psi$, both terms favor increasing within-class similarity and decreasing across-class similarity, yielding $\alpha=1$ and $\beta=-1 /(K-1)$.

\section{Experimental setup and metrics}\label{app:exp_setup}
\subsection{Synthetic shared gaussian mixture models setup}
We generate a synthetic dataset following a shared Gaussian Mixture Model (GMM) structure in \cite{bansal2024understanding}. Specifically, it creates \(K\) equally weighted Gaussian components in a \(d\)-dimensional space, each containing \(k_{\text{per\_class}}\) samples. The class means lie in a \((K-1)\)-dimensional subspace whose sum is zero, ensuring the dataset is centered. All components share the same covariance: isotropic with unit variance in the mean subspace and scaled by a factor \(\kappa\) in the orthogonal complement, producing a ``parallel-pancakes'' geometry. Consequently, the parameter \(\kappa\) controls the spread orthogonal to the subspace—small values yield flatter, well-separated clusters, while large values produce thicker, overlapping ones. This setup provides a controlled, high-dimensional benchmark for testing representation learning or optimal transport algorithms. For the specific experimental settings, we set $K \in \{10,20,50\}$; $k_{\text{per\_class}} = 50$; $d = 100$ and $\kappa = 5$ (representing a relatively large value of $\kappa$). For the encoder, we use a very \textit{simple encoder} using a lightweight three-layer MLP. The batch sampler is taken as in Algorithm~\ref{alg:4}. Algorithm~\ref{alg:4} is used only in the controlled synthetic experiment, where exact class balance is needed to directly test the theorem. It is not used in the vision results later; those experiments use the same ordinary shuffled mini-batch construction for every method.

In the results of the main text, $\psi$ is taken to be the $\psi(t) = -t$. We also report the results when using the negative-log-sigmoid function $\psi(t) = \ln(1+e^{-t})$ in the next section of the appendix. The hyperparameters are taken as follows: $\varepsilon = 0.5$, the number of Sinkhorn iterations is $n_{it} = 10$, and the learning rate is $0.0008$. All other parameters for BYOL and VicReg follow those specified in their original papers. Details about the metrics:

\begin{enumerate}
    \item Neural Collapse
        \begin{itemize}
            \item \textbf{NC1: Variability Collapse vs epochs}:
            $ \sum_{\mathfrak{c}} \sum_{x \in \mathcal{I}_\mathfrak{c}} (x - \mu_\mathfrak{c})^2$.
            \item \textbf{NC2: Class Mean Geometry vs epochs} \\
            \[
            \text{std}\left( \left\{ \cos(\hat{\mu}_\mathfrak{c}, \hat{\mu}_{\mathfrak{c'}}) \right\} \right), \quad \text{avg}\left( \left| \cos(\hat{\mu}_\mathfrak{c}, \hat{\mu}_{\mathfrak{c'}}) + \frac{1}{K-1} \right| \right)  : \mathfrak{c} \neq \mathfrak{c'}
            \]
        \end{itemize}
    \item Dimensional Collapse
        \begin{itemize}
            \item \textbf{Class Mean Covariance Spectrum of the last Training Epochs} \\ 
            \[
            \Lambda_1, \Lambda_2, \dots, \Lambda_{K-1} = \text{eig}\left( \frac{\mu^\top \mu}{K-1} \right)
            \]
            where $\mu$ is the matrix of centered class means.
        \end{itemize}
\end{enumerate}

\subsection{SCL and UCL setups}
The encoder either is a ResNet encoder (ResNet18/34/50 from torchvision with the classifier removed, global average pooling, a linear projection head to $n_\text{classes}$, and $\ell_2$ normalization) or a ViT encoder (ViT-B/16 from torchvision, using \(16 \times 16\) image patches, with the original classifier replaced by a linear projection head and \(\ell_2\)-normalized outputs). Kaiming initialization is used; when a pretrained ResNet is selected, only the new head is reinitialized (the added head uses $\mathcal{N}(0,0.01)$ weights and zero bias). $\psi(t)$ is taken to be linear function $\psi(t) = -t$. The main-run defaults are number of epochs $=100$ for MNIST and SVHN and $=400$ for CIFAR-10, CIFAR-100 and Tiny-ImageNet, temperature $\tau=0.1$, $\varepsilon = \varepsilon_{+} = \varepsilon_{-} = 0.1$, Sinkhorn iterations $=10$, batch size $=256$, optimizer AdamW with learning rate $5\times10^{-3}$ and weight decay $10^{-5}$. ResNet and ViT inputs use dataset-specific normalization (CIFAR-10/100, SVHN, or ImageNet stats for Tiny-ImageNet), with pretrained backbones resized to $224\times224$ and ImageNet normalization. Linear probing uses AdamW with lr $10^{-3}$, epochs $=500$, batch size $=256$, weight decay $0$, and kNN uses $k=20$ with cosine distance.

For UCL, the SimCLR-style two-crop pipeline applies RandomResizedCrop (scale $(0.2,1.0)$), RandomHorizontalFlip, ColorJitter$(0.4,0.4,0.4,0.1)$ with $p=0.8$, RandomGrayscale with $p=0.2$, GaussianBlur (kernel size $3$, $\sigma\in[0.1,2.0]$) with $p=0.5$, and normalization; MNIST instead uses RandomResizedCrop + RandomRotation$(15^\circ)$ before normalization. Dataset normalization uses CIFAR-10/100, SVHN, MNIST, or ImageNet statistics, and pretrained backbones resize inputs to $224\times224$ with ImageNet normalization. Positive pairs are formed by duplicated indices across the two views within each batch.

\subsection{CLIP setup}
Contrastive pretraining has been especially impactful in multimodal settings, where the goal is to align representations across modalities (e.g., images and text) while keeping mismatched pairs separable. The CLIP family (Contrastive Language-Image Pre-Training) \cite{radford2021learning} demonstrated that large-scale contrastive image--text pretraining can yield strong zero-shot transfer, catalyzing substantial follow-up work on scaling and improving vision--language pretrained models \cite{jia2021scaling, li2022blip}. From a statistical perspective, \cite{oko2025statistical} argue that optimizing a contrastive objective can make the learned embeddings approximate sufficient statistics for the image--text relationship. Furthermore, recent evidence \cite{mehta2025generalization} indicates that CLIP's zero-shot performance can depend strongly on the inference-time prompt distribution, and can be improved by using diverse prompts that better match the implicit ``caption dialect'' of the pretraining data.

In the setting of CLIP, $\{x_i\}_{i=1}^{N}$, $\{x'_j\}_{j=1}^{M}$ are two features sets of two modality which have pairing relationships (e.g: image and caption) and $Z=\{z_i\}_{i=1}^{N} = \{f_\theta(x_i)\}_{i=1}^{N}$ and $Z'=\{z'_j\}_{j=1}^{M} = \{g_\theta(x'_j)\}_{j=1}^{M}$ are their the embedding produced by two encoders $f_\theta$ and $g_\theta$ on hypersphere with parameter $\theta$. We try to learn $\theta$ so that a matched image-caption pair lands close together in the same vector space, while mismatched pairs land far apart. Vanilla CLIP uses two types of InfoNCE loss: one for images and one for text. (i.e. $\mathcal{L}_{\text{objective}}
=
\frac{1}{2}\left(\mathcal{L}_{\text{infoNCE-img}}+\mathcal{L}_{\text{infoNCE-text}}\right)$), so each image is trained to identify its true caption among all captions in the batch, and each caption is trained to identify its true image among all images in the batch. Meanwhile, from the OT-based perspective, we can directly extend the formulations in \ref{eq:neg_mmot_final} and \ref{eq:clot_final} to the multimodal setting—specifically, to CLIP—since the underlying framework can be interpreted as an unconstrained feature model and the only thing to do is which encoder embeds which features externally.

For implementation, the image encoder is a Modified ResNet-50 (layers [3,4,6,3], width 64) with an attention pooling head to produce 1024‑dim image embeddings at 224px resolution, while the text encoder is a causal Transformer with token and positional embeddings (context length 77, vocab size 49,408, width 512, 8 heads, 6 layers) projecting into the same 1024‑dim joint space for CL. Also, $\psi(t) = -t$, number of epochs $=35$, temperature $\tau=0.1$, $\varepsilon = \varepsilon_{+} = \varepsilon_{-} = 0.1$, Sinkhorn iterations $=10$, batch size $=64$, optimizer AdamW with learning rate $5\times10^{-4}$ and weight decay $0.1$.

In the retrieval task, we fix an image and retrieve its matching caption based on similarity in the representation space (and vice versa, fixing a caption to retrieve the corresponding image). 

\newpage
\section{Experimental results} \label{app:add_exp_results}
\subsection{Synthetic data}
\noindent\textbf{Visualization of representation features by tSNE projection on $\mathbb{R}^2$ when $\psi(t) = -t$}

\begin{figure}[!htp]
    \centering
    \includegraphics[width=1\linewidth]{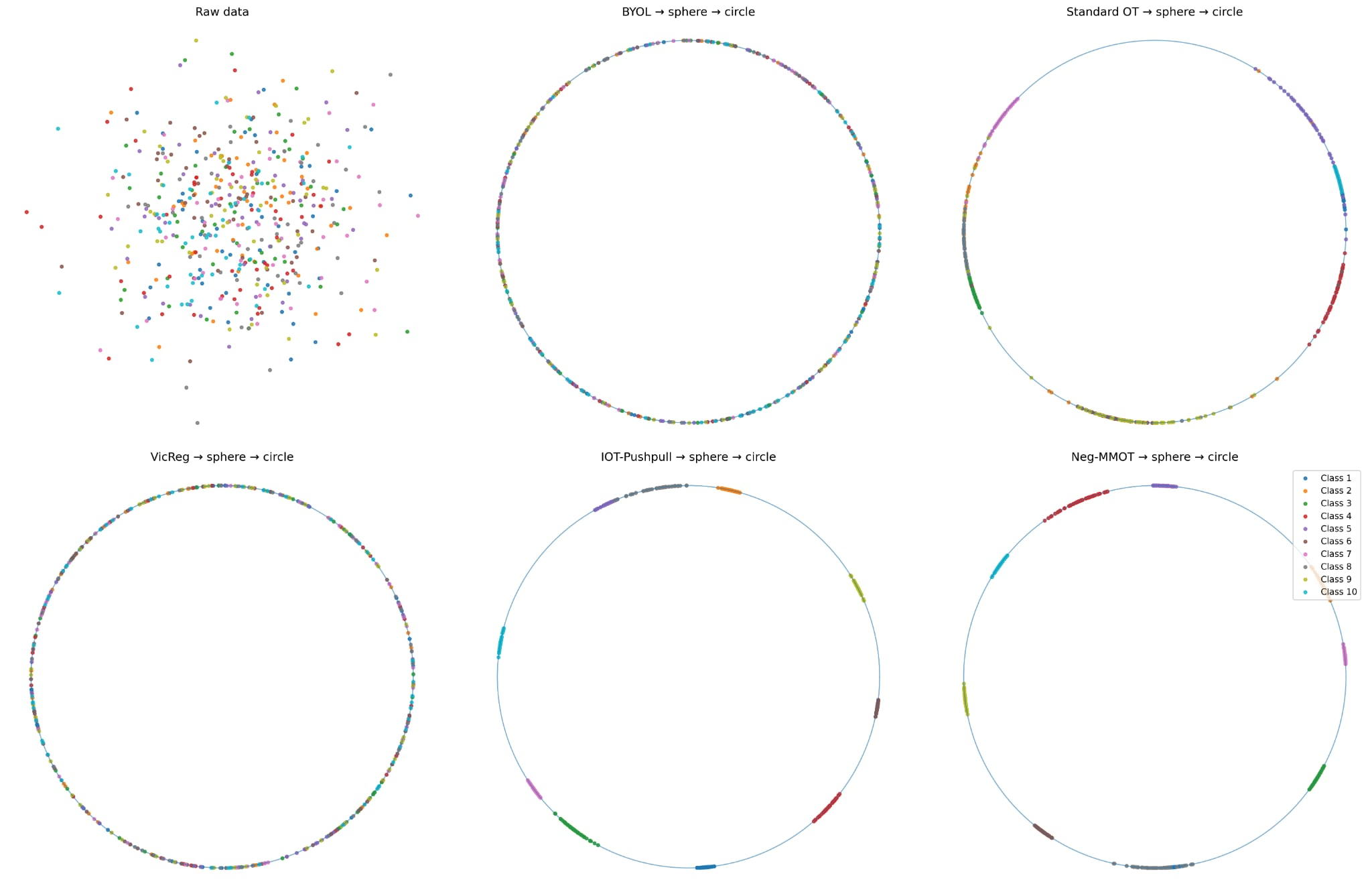}
    \caption{Visualization on $\mathbb{R}^2$. We use tSNE method for projection.}
    \label{fig:visual_final_synthetic}
\end{figure}

\noindent\textbf{Similar results for $\psi$ as negative log sigmoid}
\begin{figure}[!htp]
    \centering
    \includegraphics[width=0.8\linewidth]{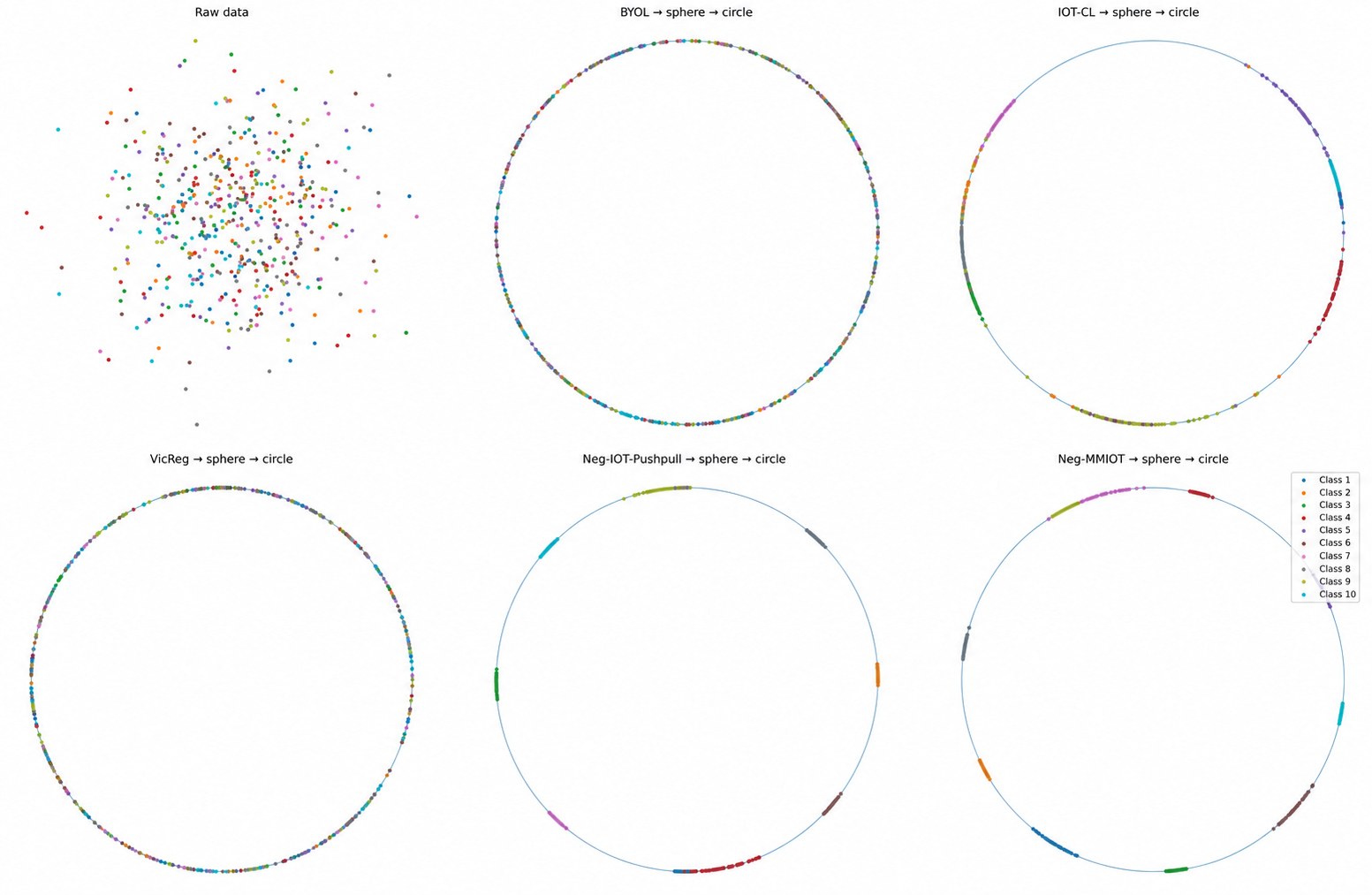}
    \caption{Visualization on $\mathbb{R}^2$. We use tSNE method for projection.}
    \label{fig:visual_final_synthetic}
\end{figure}

\begin{figure}[!htp]
    \centering
    \includegraphics[width=1\linewidth]{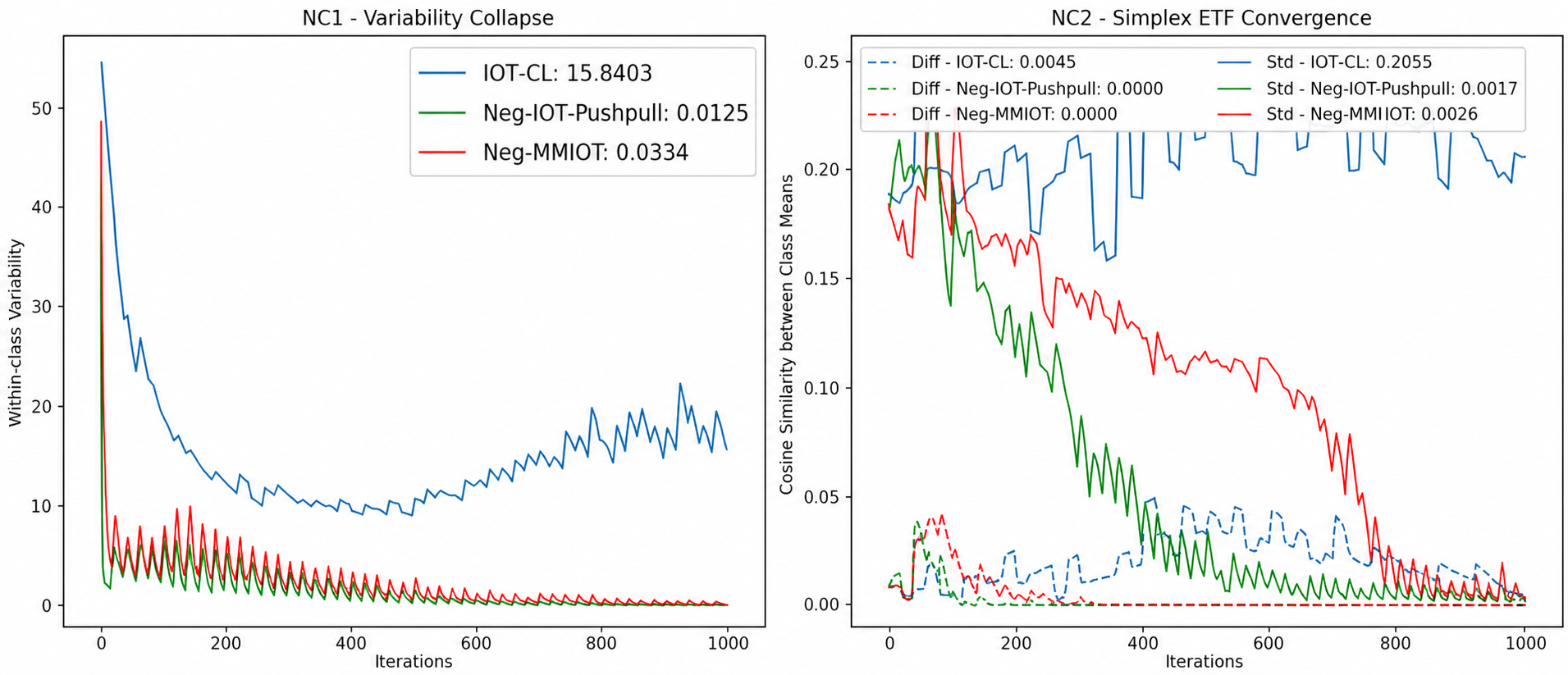}
    \caption{Neural Collapse Observation when $\psi$ is negative log sigmoid}
    \label{fig:nc_synthetic}
    \vspace{-0.2in}
\end{figure}
% \vglue -1cm
\vspace{-0.2in}
\begin{figure}[htp!]
    \centering
    \subfloat{\includegraphics[width=0.33\linewidth]{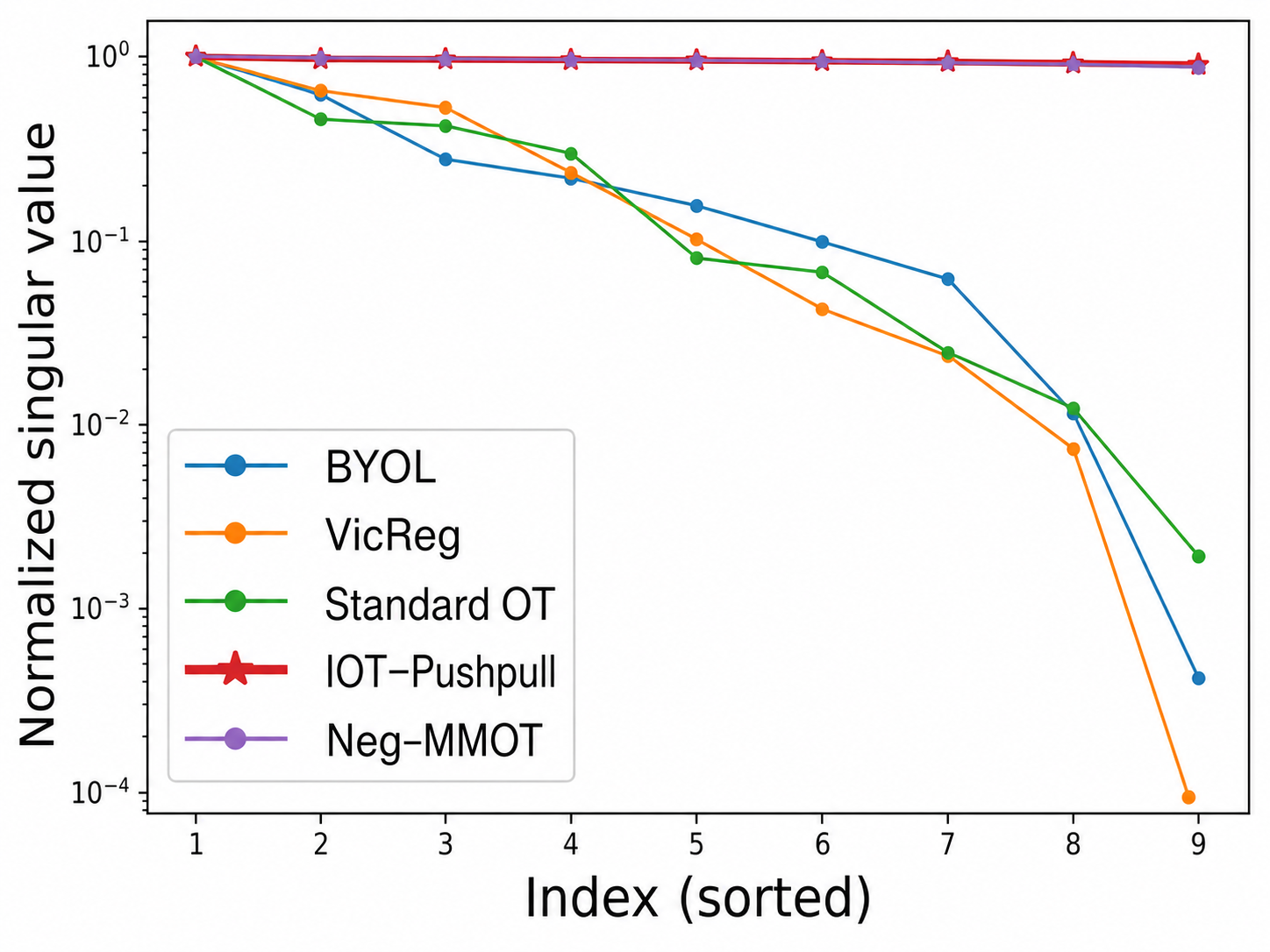}\label{rate}}
    \hfill
    \subfloat{\includegraphics[width=0.33\linewidth]{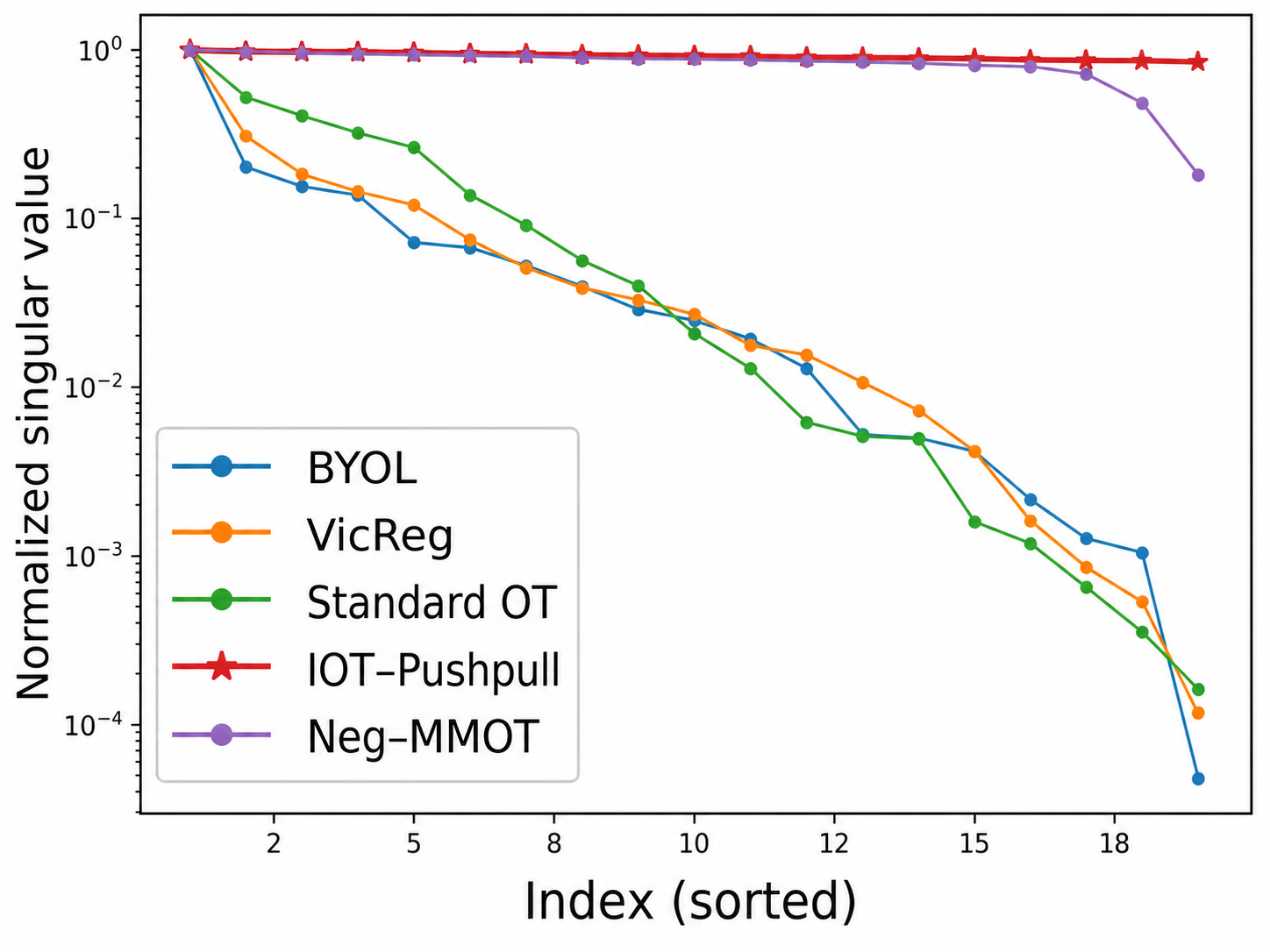}\label{rate}}
    \subfloat{\includegraphics[width=0.33\linewidth]{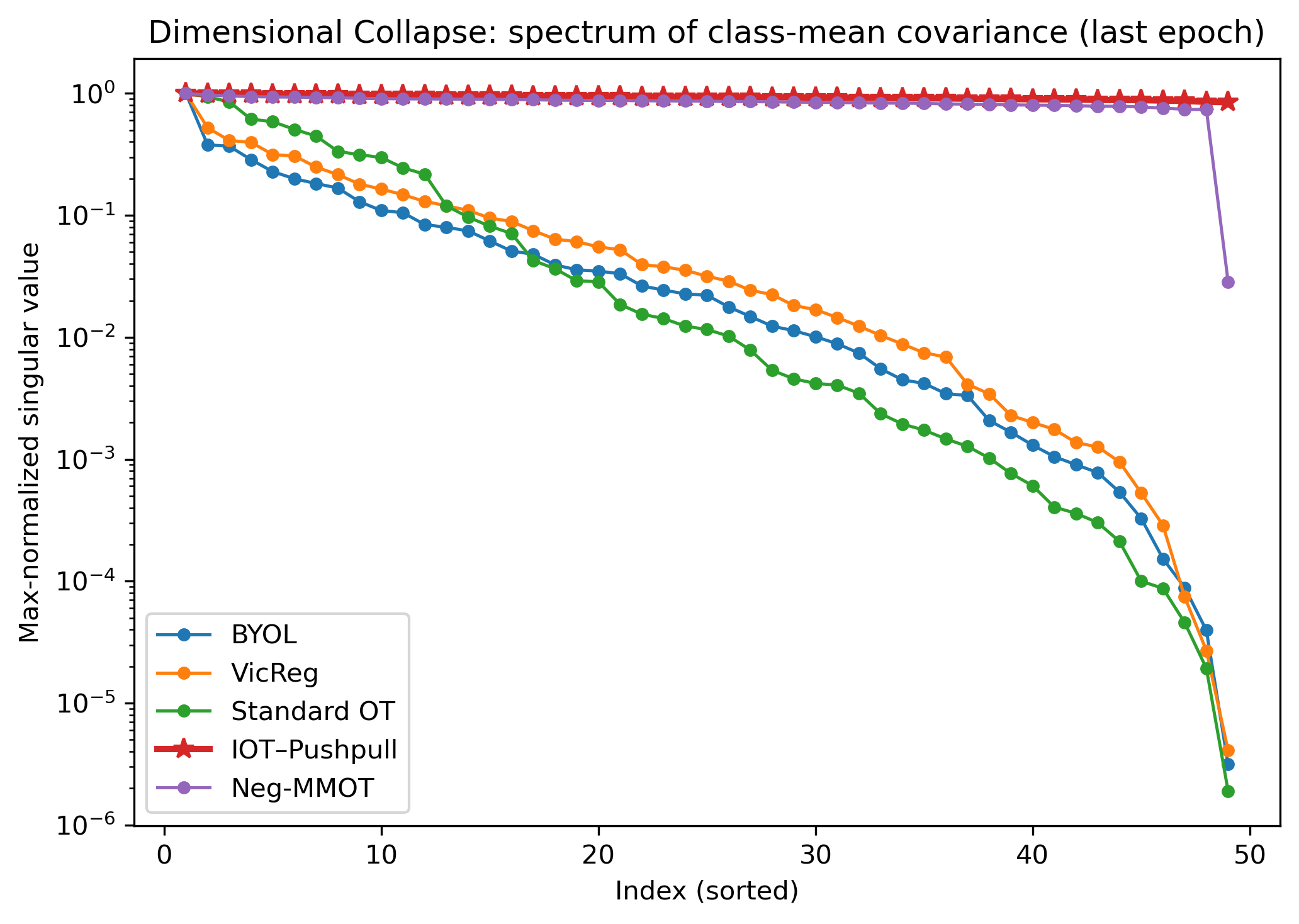}\label{rate}}
    \caption{Dimensional Collapse in the settings of different numbers of classes when $\psi$ is negative log sigmoid}
    \label{fig:dc_synthetic}
\end{figure}

\clearpage
\subsection{Vision benchmark} \label{app:detail_result_vision}
\subsubsection{SCL}
% ==================== MNIST ====================
\begin{table}[htp!]
\centering
\captionsetup{skip=15pt}
\renewcommand{\arraystretch}{1.6}
\resizebox{\linewidth}{!}{%
\begin{tabular}{|c|cc|cc|cc|cc|}
\hline
\multicolumn{1}{|c|}{SCL}
& \multicolumn{8}{c|}{\textbf{MNIST}} \\
\cline{2-9}
& \multicolumn{2}{c|}{ResNet18}
& \multicolumn{2}{c|}{ResNet34}
& \multicolumn{2}{c|}{ResNet50}
& \multicolumn{2}{c|}{ViT-B/16} \\
\cline{2-9}
& Linear & k-NN
& Linear & k-NN
& Linear & k-NN
& Linear & k-NN \\
\hline
infoNCE
& \synres{99.25}{0.10} & \synres{99.24}{0.13}
& \synres{99.27}{0.04} & \synres{99.27}{0.04}
& \synres{99.36}{0.07} & \synres{99.30}{0.12}
& \synres{99.49}{0.13} & \synres{99.46}{0.10} \\
\hline
InvaSpread
& \synres{99.27}{0.10} & \synres{99.03}{0.13}
& \synres{98.91}{0.08} & \synres{98.76}{0.13}
& \synres{98.98}{0.09} & \synres{99.17}{0.07}
& \synres{99.61}{0.08} & \synres{99.59}{0.03} \\
\hline
Standard OT
& \synres{99.22}{0.04} & \synres{99.20}{0.05}
& \synres{99.18}{0.10} & \synres{99.18}{0.13}
& \synres{99.21}{0.04} & \synres{99.17}{0.04}
& \synres{99.60}{0.05} & \synres{99.58}{0.09} \\
\hline
Neg-IOT-CL-PushPull (Ours)
& \synbest{99.33}{0.06} & \synbest{99.35}{0.10}
& \synbest{99.33}{0.11} & \synbest{99.31}{0.10}
& \synres{99.30}{0.13} & \synres{99.29}{0.07}
& \synres{99.64}{0.10} & \synres{99.64}{0.05} \\
\hline
Neg-MMIOT-CL (Ours)
& \synres{99.21}{0.09} & \synres{99.20}{0.10}
& \synres{99.08}{0.03} & \synres{99.02}{0.10}
& \synbest{99.41}{0.13} & \synbest{99.41}{0.14}
& \synbest{99.66}{0.11} & \synbest{99.67}{0.10} \\
\hline
\end{tabular}%
}
\caption{SCL results on MNIST with different backbones.}
\label{tab:scl-mnist}
\end{table}

% ==================== SVHN ====================
\begin{table}[htp!]
\centering
\captionsetup{skip=5pt}
\renewcommand{\arraystretch}{1.5}
\resizebox{\linewidth}{!}{%
\begin{tabular}{|c|cc|cc|cc|cc|}
\hline
\multicolumn{1}{|c|}{SCL}
& \multicolumn{8}{c|}{SVHN} \\
\cline{2-9}
& \multicolumn{2}{c|}{ResNet18}
& \multicolumn{2}{c|}{ResNet34}
& \multicolumn{2}{c|}{ResNet50}
& \multicolumn{2}{c|}{ViT-B/16} \\
\cline{2-9}
& Linear & k-NN
& Linear & k-NN
& Linear & k-NN
& Linear & k-NN \\
\hline
infoNCE
& \synres{91.76}{0.33} & \synres{91.71}{0.21}
& \synbest{91.73}{0.27} & \synbest{91.77}{0.43}
& \synres{91.96}{0.42} & \synres{91.94}{0.29}
& \synres{94.07}{0.38} & \synres{94.04}{0.37} \\
\hline
InvaSpread
& \synres{86.55}{0.28} & \synres{87.21}{0.36}
& \synres{82.65}{0.34} & \synres{85.94}{0.23}
& \synres{81.79}{0.30} & \synres{80.80}{0.30}
& \synres{85.35}{0.45} & \synres{84.78}{0.48} \\
\hline
Standard OT
& \synres{91.24}{0.32} & \synres{91.16}{0.18}
& \synres{91.58}{0.13} & \synres{91.61}{0.15}
& \synres{92.08}{0.25} & \synres{91.99}{0.31}
& \synres{94.24}{0.47} & \synres{94.12}{0.36} \\
\hline
Neg-IOT-CL-PushPull (Ours)
& \synres{91.03}{0.14} & \synres{90.99}{0.27}
& \synres{91.55}{0.27} & \synres{91.56}{0.13}
& \synres{91.89}{0.14} & \synres{91.83}{0.41}
& \synres{94.38}{0.33} & \synres{94.38}{0.35} \\
\hline
Neg-MMIOT-CL (Ours)
& \synbest{91.91}{0.32} & \synbest{91.96}{0.15}
& \synres{91.52}{0.47} & \synres{91.46}{0.28}
& \synbest{92.74}{0.46} & \synbest{92.71}{0.16}
& \synbest{94.87}{0.46} & \synbest{94.80}{0.19} \\
\hline
\end{tabular}%
}
\caption{SCL results on SVHN with different backbones. }
\label{tab:scl-svhn}
\end{table}

% ==================== CIFAR-10 ====================
\begin{table}[htp!]
\centering
\captionsetup{skip=5pt}
\renewcommand{\arraystretch}{1.5}
\resizebox{\linewidth}{!}{%
\begin{tabular}{|c|cc|cc|cc|cc|}
\hline
\multicolumn{1}{|c|}{SCL}
& \multicolumn{8}{c|}{CIFAR-10} \\
\cline{2-9}
& \multicolumn{2}{c|}{ResNet18}
& \multicolumn{2}{c|}{ResNet34}
& \multicolumn{2}{c|}{ResNet50}
& \multicolumn{2}{c|}{ViT-B/16} \\
\cline{2-9}
& Linear & k-NN
& Linear & k-NN
& Linear & k-NN
& Linear & k-NN \\
\hline
infoNCE
& \synres{85.01}{0.25} & \synres{84.82}{0.47}
& \synres{86.21}{0.27} & \synres{86.06}{0.29}
& \synbest{91.39}{0.34} & \synbest{90.98}{0.43}
& \synres{90.40}{0.38} & \synres{90.32}{0.46} \\
\hline
InvaSpread
& \synres{73.01}{0.42} & \synres{76.01}{0.49}
& \synres{75.24}{0.33} & \synres{75.59}{0.31}
& \synres{78.46}{0.40} & \synres{77.53}{0.47}
& \synres{72.62}{0.24} & \synres{72.41}{0.56} \\
\hline
Standard OT
& \synres{83.76}{0.30} & \synres{83.47}{0.34}
& \synres{83.29}{0.24} & \synres{84.40}{0.41}
& \synres{90.78}{0.55} & \synres{90.80}{0.48}
& \synres{84.72}{0.29} & \synres{84.55}{0.57} \\
\hline
Neg-IOT-CL-PushPull (Ours)
& \synres{82.15}{0.20} & \synres{82.66}{0.54}
& \synres{83.23}{0.20} & \synres{82.21}{0.36}
& \synres{89.59}{0.55} & \synres{89.62}{0.55}
& \synres{88.69}{0.53} & \synres{88.54}{0.16} \\
\hline
Neg-MMIOT-CL (Ours)
& \synbest{88.38}{0.40} & \synbest{88.21}{0.43}
& \synbest{89.39}{0.15} & \synbest{89.33}{0.47}
& \synres{90.64}{0.35} & \synres{90.53}{0.18}
& \synbest{91.97}{0.50} & \synbest{91.85}{0.53} \\
\hline
\end{tabular}%
}
\caption{SCL results on CIFAR-10 with different backbones. }
\label{tab:scl-CIFAR-10}
\end{table}

% ==================== CIFAR-100 ====================
\begin{table}[htp!]
\centering
\captionsetup{skip=5pt}
\renewcommand{\arraystretch}{1.5}
\resizebox{\linewidth}{!}{%
\begin{tabular}{|c|cc|cc|cc|cc|}
\hline
\multicolumn{1}{|c|}{SCL}
& \multicolumn{8}{c|}{CIFAR-100} \\
\cline{2-9}
& \multicolumn{2}{c|}{ResNet18}
& \multicolumn{2}{c|}{ResNet34}
& \multicolumn{2}{c|}{ResNet50}
& \multicolumn{2}{c|}{ViT-B/16} \\
\cline{2-9}
& Linear & k-NN
& Linear & k-NN
& Linear & k-NN
& Linear & k-NN \\
\hline
infoNCE
& \synres{69.72}{0.58} & \synres{69.46}{0.39}
& \synres{70.61}{0.33} & \synres{70.33}{0.50}
& \synres{74.01}{0.35} & \synres{73.56}{0.40}
& \synres{73.59}{0.33} & \synres{73.58}{0.34} \\
\hline
InvaSpread
& \synres{55.46}{0.53} & \synres{54.32}{0.81}
& \synres{64.84}{0.62} & \synres{63.10}{0.46}
& \synres{63.26}{0.81} & \synres{62.24}{0.34}
& \synres{67.08}{0.49} & \synres{66.68}{0.30} \\
\hline
Standard OT
& \synres{69.24}{0.56} & \synres{69.20}{0.58}
& \synres{70.53}{0.82} & \synres{70.41}{0.72}
& \synres{72.71}{0.30} & \synres{72.88}{0.34}
& \synres{73.49}{0.56} & \synres{73.59}{0.70} \\
\hline
Neg-IOT-CL-PushPull (Ours)
& \synres{69.26}{0.29} & \synres{69.46}{0.52}
& \synres{70.76}{0.76} & \synres{70.56}{0.74}
& \synres{74.80}{0.75} & \synres{74.19}{0.30}
& \synbest{74.67}{0.69} & \synbest{74.83}{0.67} \\
\hline
Neg-MMIOT-CL (Ours)
& \synbest{71.88}{0.46} & \synbest{71.59}{0.50}
& \synbest{72.15}{0.43} & \synbest{71.88}{0.29}
& \synbest{75.17}{0.75} & \synbest{75.17}{0.40}
& \synres{74.62}{0.54} & \synres{74.53}{0.81} \\
\hline
\end{tabular}%
}
\caption{SCL results on CIFAR-100 with different backbones. }
\label{tab:scl-CIFAR-100}
\end{table}

% ==================== Tiny-ImageNet ====================
\begin{table}[htp!]
\centering
\captionsetup{skip=15pt}
\renewcommand{\arraystretch}{1.6}
\resizebox{\linewidth}{!}{%
\begin{tabular}{|c|cc|cc|cc|cc|}
\hline
\multicolumn{1}{|c|}{SCL}
& \multicolumn{8}{c|}{\textbf{Tiny-ImageNet}} \\
\cline{2-9}
& \multicolumn{2}{c|}{ResNet18}
& \multicolumn{2}{c|}{ResNet34}
& \multicolumn{2}{c|}{ResNet50}
& \multicolumn{2}{c|}{ViT-B/16} \\
\cline{2-9}
& Linear & k-NN
& Linear & k-NN
& Linear & k-NN
& Linear & k-NN \\
\hline
infoNCE
& \synres{54.56}{0.65} & \synres{54.34}{0.55}
& \synres{55.02}{0.60} & \synres{54.52}{0.54}
& \synres{62.60}{0.62} & \synbest{63.30}{1.00}
& \synres{46.32}{0.86} & \synres{46.37}{0.68} \\
\hline
InvaSpread
& \synres{51.84}{0.84} & \synres{51.26}{0.42}
& \synres{51.20}{0.44} & \synres{49.98}{0.38}
& \synres{55.06}{0.67} & \synres{53.72}{0.56}
& \synres{44.19}{0.91} & \synres{43.33}{0.82} \\
\hline
Standard OT
& \synres{54.64}{0.86} & \synres{53.92}{0.61}
& \synres{55.44}{0.36} & \synres{55.08}{0.92}
& \synres{61.20}{0.83} & \synres{60.84}{0.36}
& \synres{45.27}{0.80} & \synres{45.19}{0.41} \\
\hline
Neg-IOT-CL-PushPull (Ours)
& \synbest{57.28}{0.74} & \synbest{56.88}{0.47}
& \synbest{56.74}{0.80} & \synbest{56.38}{0.79}
& \synres{62.04}{0.62} & \synres{61.42}{1.02}
& \synres{46.54}{0.54} & \synres{46.58}{0.80} \\
\hline
Neg-MMIOT-CL (Ours)
& \synres{56.00}{1.02} & \synres{55.80}{0.40}
& \synres{56.02}{0.96} & \synres{55.92}{0.44}
& \synbest{63.54}{0.81} & \synres{63.20}{0.40}
& \synbest{48.02}{0.79} & \synbest{47.96}{0.94} \\
\hline
\end{tabular}%
}
\caption{SCL results on Tiny-ImageNet with different backbones. }
\label{tab:scl-tiny-imagenet}
\end{table}

\subsubsection{UCL}
% ==================== MNIST ====================
\begin{table}[htp!]
\centering
\captionsetup{skip=15pt}
\renewcommand{\arraystretch}{1.6}
\resizebox{\linewidth}{!}{%
\begin{tabular}{|c|cc|cc|cc|cc|}
\hline
\multicolumn{1}{|c|}{UCL}
& \multicolumn{8}{c|}{\textbf{MNIST}} \\
\cline{2-9}
& \multicolumn{2}{c|}{ResNet18}
& \multicolumn{2}{c|}{ResNet34}
& \multicolumn{2}{c|}{ResNet50}
& \multicolumn{2}{c|}{ViT-B/16} \\
\cline{2-9}
& Linear & k-NN
& Linear & k-NN
& Linear & k-NN
& Linear & k-NN \\
\hline
infoNCE
& \synres{97.63}{0.08} & \synres{98.56}{0.05}
& \synres{97.43}{0.09} & \synres{98.22}{0.05}
& \synres{97.73}{0.08} & \synres{98.20}{0.05}
& \synres{98.40}{0.16} & \synres{98.82}{0.12} \\
\hline
InvaSpread
& \synres{98.00}{0.18} & \synres{97.37}{0.18}
& \synres{97.42}{0.05} & \synres{96.93}{0.10}
& \synres{97.65}{0.15} & \synres{96.89}{0.06}
& \synres{97.95}{0.14} & \synres{97.70}{0.13} \\
\hline
Standard OT
& \synres{97.70}{0.09} & \synres{98.07}{0.06}
& \synbest{97.60}{0.13} & \synbest{98.39}{0.18}
& \synres{97.20}{0.06} & \synres{98.09}{0.11}
& \synres{98.35}{0.05} & \synres{98.74}{0.13} \\
\hline
Neg-IOT-CL-PushPull (Ours)
& \synbest{97.94}{0.11} & \synbest{98.62}{0.08}
& \synres{97.49}{0.13} & \synres{98.31}{0.05}
& \synres{96.56}{0.14} & \synres{98.07}{0.13}
& \synres{98.61}{0.14} & \synres{98.91}{0.12} \\
\hline
Neg-MMIOT-CL (Ours)
& \synres{97.60}{0.05} & \synres{98.43}{0.16}
& \synres{97.52}{0.16} & \synres{98.19}{0.17}
& \synbest{97.73}{0.18} & \synbest{98.42}{0.08}
& \synbest{98.72}{0.11} & \synbest{99.03}{0.06} \\
\hline
\end{tabular}%
}
\caption{UCL results on MNIST with different backbones. }
\label{tab:ucl-mnist}
\end{table}

% ==================== SVHN ====================
\begin{table}[htp!]
\centering
\captionsetup{skip=5pt}
\renewcommand{\arraystretch}{1.5}
\resizebox{\linewidth}{!}{%
\begin{tabular}{|c|cc|cc|cc|cc|}
\hline
\multicolumn{1}{|c|}{UCL}
& \multicolumn{8}{c|}{SVHN} \\
\cline{2-9}
& \multicolumn{2}{c|}{ResNet18}
& \multicolumn{2}{c|}{ResNet34}
& \multicolumn{2}{c|}{ResNet50}
& \multicolumn{2}{c|}{ViT-B/16} \\
\cline{2-9}
& Linear & k-NN
& Linear & k-NN
& Linear & k-NN
& Linear & k-NN \\
\hline
infoNCE
& \synres{77.44}{0.60} & \synres{81.21}{0.39}
& \synres{77.70}{0.46} & \synres{81.86}{0.63}
& \synres{78.13}{0.36} & \synres{83.00}{0.43}
& \synres{84.67}{0.57} & \synres{86.24}{0.21} \\
\hline
InvaSpread
& \synres{74.36}{0.53} & \synres{71.49}{0.37}
& \synres{72.78}{0.46} & \synres{69.47}{0.47}
& \synres{73.48}{0.63} & \synres{70.53}{0.38}
& \synres{75.13}{0.38} & \synres{73.88}{0.33} \\
\hline
Standard OT
& \synres{74.32}{0.29} & \synres{79.42}{0.59}
& \synres{75.65}{0.52} & \synres{80.12}{0.61}
& \synres{75.96}{0.37} & \synres{79.69}{0.38}
& \synres{83.94}{0.33} & \synres{85.76}{0.23} \\
\hline
Neg-IOT-CL-PushPull (Ours)
& \synres{79.28}{0.25} & \synres{81.79}{0.44}
& \synres{79.42}{0.43} & \synres{81.38}{0.34}
& \synres{78.63}{0.55} & \synres{80.85}{0.38}
& \synres{87.45}{0.30} & \synres{87.02}{0.26} \\
\hline
Neg-MMIOT-CL (Ours)
& \synbest{85.26}{0.48} & \synbest{86.49}{0.42}
& \synbest{85.88}{0.41} & \synbest{87.75}{0.51}
& \synbest{84.65}{0.50} & \synbest{88.09}{0.46}
& \synbest{88.61}{0.22} & \synbest{89.28}{0.40} \\
\hline
\end{tabular}%
}
\caption{UCL results on SVHN with different backbones. }
\label{tab:ucl-svhn}
\end{table}

% ==================== CIFAR-10 ====================
\begin{table}[htp!]
\centering
\captionsetup{skip=5pt}
\renewcommand{\arraystretch}{1.5}
\resizebox{\linewidth}{!}{%
\begin{tabular}{|c|cc|cc|cc|cc|}
\hline
\multicolumn{1}{|c|}{UCL}
& \multicolumn{8}{c|}{CIFAR-10} \\
\cline{2-9}
& \multicolumn{2}{c|}{ResNet18}
& \multicolumn{2}{c|}{ResNet34}
& \multicolumn{2}{c|}{ResNet50}
& \multicolumn{2}{c|}{ViT-B/16} \\
\cline{2-9}
& Linear & k-NN
& Linear & k-NN
& Linear & k-NN
& Linear & k-NN \\
\hline
infoNCE
& \synres{75.30}{0.65} & \synres{76.00}{0.57}
& \synres{75.45}{0.30} & \synres{76.47}{0.51}
& \synres{76.81}{0.52} & \synres{76.69}{0.62}
& \synres{84.57}{0.53} & \synres{85.84}{0.61} \\
\hline
InvaSpread
& \synres{67.27}{0.45} & \synres{65.83}{0.74}
& \synres{65.92}{0.38} & \synres{64.54}{0.65}
& \synres{64.88}{0.52} & \synres{63.69}{0.30}
& \synres{64.13}{0.49} & \synres{63.28}{0.64} \\
\hline
Standard OT
& \synres{75.33}{0.62} & \synres{75.27}{0.34}
& \synres{75.94}{0.72} & \synres{76.77}{0.62}
& \synres{78.12}{0.57} & \synres{77.67}{0.35}
& \synres{83.22}{0.43} & \synres{86.28}{0.35} \\
\hline
Neg-IOT-CL-PushPull (Ours)
& \synbest{77.43}{0.25} & \synbest{76.05}{0.52}
& \synbest{77.02}{0.65} & \synbest{76.34}{0.77}
& \synbest{78.18}{0.63} & \synbest{77.89}{0.29}
& \synbest{87.56}{0.23} & \synbest{86.96}{0.36} \\
\hline
Neg-MMIOT-CL (Ours)
& \synres{75.89}{0.68} & \synres{74.89}{0.53}
& \synres{75.53}{0.36} & \synres{74.69}{0.33}
& \synres{77.01}{0.52} & \synres{76.41}{0.62}
& \synres{86.15}{0.72} & \synres{85.95}{0.61} \\
\hline
\end{tabular}%
}
\caption{UCL results on CIFAR-10 with different backbones. }
\label{tab:ucl-CIFAR-10}
\end{table}

% ==================== CIFAR-100 ====================
\begin{table}[htp!]
\centering
\captionsetup{skip=5pt}
\renewcommand{\arraystretch}{1.5}
\resizebox{\linewidth}{!}{%
\begin{tabular}{|c|cc|cc|cc|cc|}
\hline
\multicolumn{1}{|c|}{UCL}
& \multicolumn{8}{c|}{CIFAR-100} \\
\cline{2-9}
& \multicolumn{2}{c|}{ResNet18}
& \multicolumn{2}{c|}{ResNet34}
& \multicolumn{2}{c|}{ResNet50}
& \multicolumn{2}{c|}{ViT-B/16} \\
\cline{2-9}
& Linear & k-NN
& Linear & k-NN
& Linear & k-NN
& Linear & k-NN \\
\hline
infoNCE
& \synres{48.34}{0.32} & \synres{49.16}{0.31}
& \synres{48.09}{0.44} & \synres{48.60}{0.48}
& \synres{49.05}{0.75} & \synres{49.51}{0.79}
& \synres{56.99}{0.89} & \synres{55.19}{0.37} \\
\hline
InvaSpread
& \synres{40.37}{0.77} & \synres{40.63}{0.82}
& \synres{44.23}{0.57} & \synres{45.45}{0.55}
& \synres{44.51}{0.44} & \synres{45.83}{0.44}
& \synres{44.17}{0.49} & \synres{39.59}{0.36} \\
\hline
Standard OT
& \synres{46.15}{0.38} & \synres{46.04}{0.31}
& \synres{46.35}{0.88} & \synres{45.54}{0.41}
& \synres{45.54}{0.91} & \synres{45.60}{0.84}
& \synres{55.46}{0.62} & \synres{54.21}{0.35} \\
\hline
Neg-IOT-CL-PushPull (Ours)
& \synbest{51.81}{0.79} & \synbest{48.41}{0.73}
& \synbest{51.04}{0.40} & \synbest{48.13}{0.44}
& \synbest{52.20}{0.71} & \synbest{49.13}{0.52}
& \synres{55.45}{0.88} & \synres{54.52}{0.91} \\
\hline
Neg-MMIOT-CL (Ours)
& \synres{49.15}{0.51} & \synres{47.30}{0.47}
& \synres{48.24}{0.53} & \synres{46.57}{0.65}
& \synres{49.27}{0.61} & \synres{47.73}{0.68}
& \synbest{57.52}{0.60} & \synbest{55.44}{0.48} \\
\hline
\end{tabular}%
}
\caption{UCL results on CIFAR-100 with different backbones. }
\label{tab:ucl-CIFAR-100}
\end{table}

% ==================== Tiny-ImageNet ====================
\begin{table}[htp!]
\centering
\captionsetup{skip=15pt}
\renewcommand{\arraystretch}{1.6}
\resizebox{\linewidth}{!}{%
\begin{tabular}{|c|cc|cc|cc|cc|}
\hline
\multicolumn{1}{|c|}{UCL}
& \multicolumn{8}{c|}{\textbf{Tiny-ImageNet}} \\
\cline{2-9}
& \multicolumn{2}{c|}{ResNet18}
& \multicolumn{2}{c|}{ResNet34}
& \multicolumn{2}{c|}{ResNet50}
& \multicolumn{2}{c|}{ViT-B/16} \\
\cline{2-9}
& Linear & k-NN
& Linear & k-NN
& Linear & k-NN
& Linear & k-NN \\
\hline
infoNCE
& \synres{53.68}{0.69} & \synres{50.35}{0.44}
& \synres{50.90}{0.96} & \synres{48.34}{0.77}
& \synres{55.23}{1.18} & \synres{53.56}{0.92}
& \synres{51.45}{0.46} & \synres{49.88}{0.48} \\
\hline
InvaSpread
& \synres{48.20}{0.92} & \synres{46.61}{0.86}
& \synres{46.14}{0.76} & \synres{42.51}{1.15}
& \synres{46.85}{1.08} & \synres{43.51}{0.62}
& \synres{43.18}{0.52} & \synres{39.66}{1.15} \\
\hline
Standard OT
& \synres{50.45}{0.97} & \synres{46.68}{0.88}
& \synres{51.78}{0.47} & \synres{48.75}{1.12}
& \synres{55.29}{0.70} & \synres{52.92}{1.15}
& \synres{53.12}{0.43} & \synres{51.44}{0.79} \\
\hline
Neg-IOT-CL-PushPull (Ours)
& \synres{50.88}{0.47} & \synres{47.77}{0.66}
& \synres{51.70}{1.18} & \synres{48.56}{0.96}
& \synres{55.37}{0.86} & \synres{52.17}{0.65}
& \synres{54.35}{1.02} & \synres{52.90}{1.06} \\
\hline
Neg-MMIOT-CL (Ours)
& \synbest{54.28}{1.00} & \synbest{52.84}{1.14}
& \synbest{57.03}{1.06} & \synbest{58.61}{0.60}
& \synbest{60.27}{0.80} & \synbest{60.33}{0.96}
& \synbest{57.48}{0.47} & \synbest{56.92}{0.83} \\
\hline
\end{tabular}%
}
\caption{UCL results on Tiny-ImageNet with different backbones. }
\label{tab:ucl-tiny-imagenet}
\end{table}

\newpage 
\subsection{Ablation study} \label{sec:ablation}
\paragraph{Ablation Study with different hyperparameters}
\begin{table}[htp!]
\centering
\captionsetup{skip=10pt}
\renewcommand{\arraystretch}{1.4}
\setlength{\tabcolsep}{4pt}
\small

% Learning rate
\begin{minipage}[t]{0.29\linewidth}
\vspace{0pt}
\centering
\resizebox{\linewidth}{!}{%
\begin{tabular}{|c|c|c|}
\hline
\textbf{Learning Rate} & \textbf{Linear} & \textbf{k-NN} \\
\hline
$5 \times 10^{-4}$ & 79.52 & 79.40 \\
$1 \times 10^{-3}$ & 82.00 & 81.92 \\
$5 \times 10^{-3}$ & 82.15 & 82.66 \\
0.01 & 80.00 & 79.85 \\
0.03 & 77.70 & 77.81 \\
0.1  & 75.95 & 75.04 \\
\hline
\end{tabular}%
}
\end{minipage}\hfill
%
% Sinkhorn iterations
\begin{minipage}[t]{0.225\linewidth}
\vspace{0pt}
\centering
\resizebox{\linewidth}{!}{%
\begin{tabular}{|c|c|c|}
\hline
\textbf{S-Iter.} & \textbf{Lin.} & \textbf{k-NN} \\
\hline
1  & 81.81 & 83.59 \\
2  & 82.14 & 83.08 \\
5  & 82.28 & 83.53 \\
10 & 82.15 & 82.66 \\
\hline
\end{tabular}%
}
\end{minipage}\hfill
%
% Tau
\begin{minipage}[t]{0.225\linewidth}
\vspace{0pt}
\centering
\resizebox{\linewidth}{!}{%
\begin{tabular}{|c|c|c|}
\hline
\boldmath$\tau$ & \textbf{Lin.} & \textbf{k-NN} \\
\hline
0.01 & 88.94 & 89.49 \\
0.1  & 82.15 & 82.66 \\
1    & 82.22 & 77.61 \\
10   & 70.62 & 69.62 \\
\hline
\end{tabular}%
}
\end{minipage}\hfill
%
% Epsilon
\begin{minipage}[t]{0.225\linewidth}
\vspace{0pt}
\centering
\resizebox{\linewidth}{!}{%
\begin{tabular}{|c|c|c|}
\hline
\boldmath$\epsilon$ & \textbf{Lin.} & \textbf{k-NN} \\
\hline
0.01 & 60.78 & 60.14 \\
0.1  & 82.15 & 82.66 \\
0.5  & 79.52 & 79.52 \\
1    & 76.36 & 76.18 \\
10   & 64.10 & 64.48 \\
\hline
\end{tabular}%
}
\end{minipage}

\normalsize
\caption{Performance on CIFAR-10 of Neg-IOT-CL-PushPull under different hyperparameter settings using ResNet-18 as the backbone in SCL.}
\label{tab:hyperparameter_results}
\end{table}
We conducted an ablation study to evaluate the impact of the learning rate, the number of Sinkhorn iterations, the temperature parameter $\tau$, and entropic regularization weight $\epsilon$. As shown in \Cref{tab:hyperparameter_results}, Neg-IOT-CL-PushPull performs consistently across a wide range of hyperparameter settings. For the entropic regularization parameter, the table shows that the downstream performance varies with the choice of $\epsilon$, as is commonly observed in OT-based methods. Importantly, the result is still good within a moderate neighborhood of $\epsilon$ (here, from 0.1 to 1), whereas performance degrades when $\epsilon$ becomes either excessively small or excessively large. Similar sensitivity patterns have also been reported in prior OT literature \cite{tai2021sinkhorn, liu2023bilaterally}. We additionally evaluate the standard IOT-CL formulation under the same values of $\epsilon$ and observe a qualitatively similar trend. These results suggest that sensitivity to $\epsilon$ is not specific to our method, but is an inherent characteristic of entropically regularized OT objectives.

\paragraph{Efficiency}
\begin{figure}[htp!]
    \centering
    \subfloat{\includegraphics[width=0.5\linewidth]{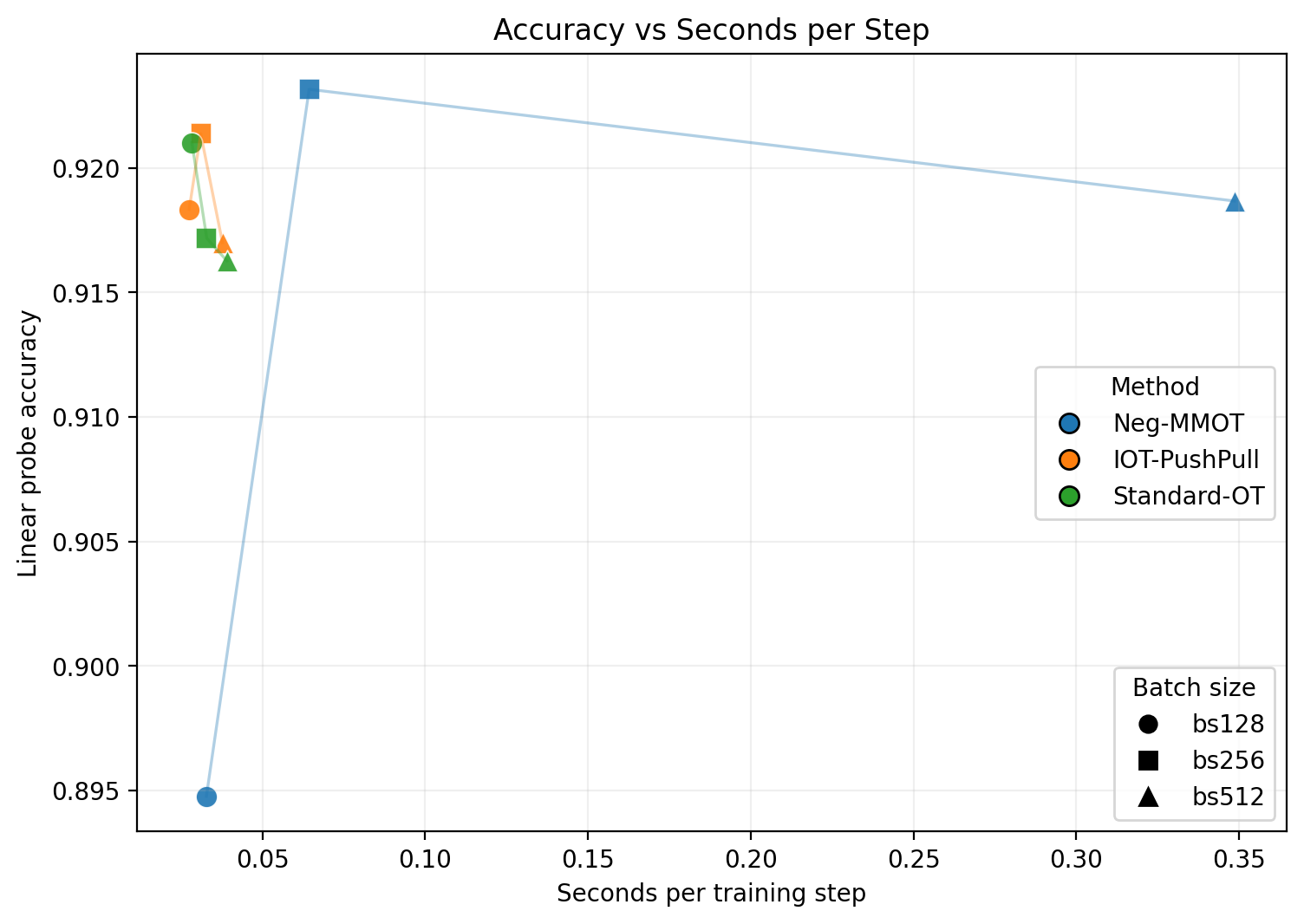}\label{pareto}}
    \hfill
    \subfloat{\includegraphics[width=0.5\linewidth]{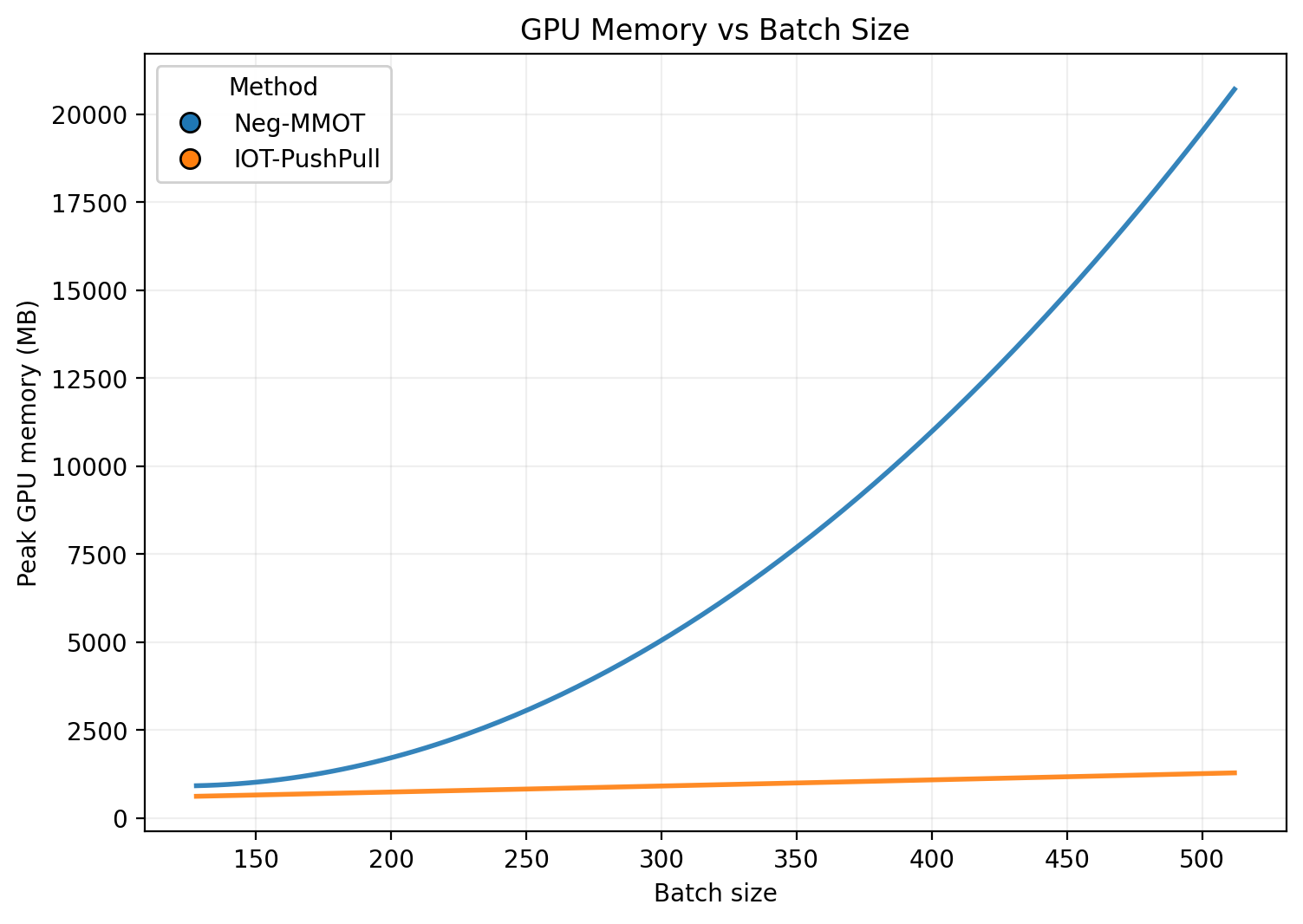}\label{gpu_trade}}
    \caption{Efficiency and scaling comparison of OT-based objectives. \textbf{Left}: linear-probe accuracy versus wall-clock seconds per training step for different batch sizes. Points closer to the upper-left indicate a better efficiency–performance tradeoff. \textbf{Right}: peak GPU memory as a function of batch size. Neg-IOT-CL-PushPull stays close to Standard OT in runtime while avoiding the steep runtime and memory growth of Neg-MMIOT-CL}
    \label{fig:efficiency_trade_off}
\end{figure}
The efficiency results in \Cref{fig:efficiency_trade_off} match the algorithmic motivation of Neg-IOT-CL-PushPull at the beginning of Section 3.2. Neg-MMIOT-CL can still be attractive when one wants the formulation with the clearest theoretical connection to the Neural Collapse analysis, but Neg-IOT-CL-PushPull is the better choice when scaling to larger batches or multimodal training budgets.

We also report the training time, measured as the average time in seconds per epoch, for the two largest datasets, CIFAR-100 and Tiny-ImageNet, as shown below.

\begin{table*}[t!]
\centering
\label{tab:training_time_scl}
\resizebox{\textwidth}{!}{
\begin{tabular}{lcccccccc}
\toprule
& \multicolumn{4}{c}{\textbf{CIFAR-100}} 
& \multicolumn{4}{c}{\textbf{Tiny-ImageNet}} \\
\cmidrule(lr){2-5} \cmidrule(lr){6-9}
\textbf{SCL}
& \textbf{ResNet18} 
& \textbf{ResNet34} 
& \textbf{ResNet50} 
& \textbf{ViT-B/16}
& \textbf{ResNet18} 
& \textbf{ResNet34} 
& \textbf{ResNet50} 
& \textbf{ViT-B/16} \\
\midrule
InfoNCE     & 3.22 & 4.40 & 5.66 & 13.23 & 6.44 & 9.37  & 11.08 & 27.68 \\
Standard OT & 3.40 & 4.86 & 6.46 & 13.59 & 6.92 & 9.35  & 11.56 & 28.39 \\
InvaSpread  & 3.16 & 4.48 & 5.85 & 13.02 & 6.87 & 9.49  & 11.81 & 28.20 \\
OT-PushPull & 3.69 & 5.62 & 7.04 & 13.87 & 7.69 & 10.19 & 12.73 & 29.64 \\
MMIOT       & 6.01 & 7.04 & 9.64 & 16.71 & 10.25 & 13.39 & 16.05 & 35.19 \\
\bottomrule
\end{tabular}
}
\caption{Average training time (seconds per epoch) for SCL methods on CIFAR-100 and Tiny-ImageNet.}
\end{table*}

\begin{table*}[htp!]
\centering
\label{tab:training_time_ucl}
\resizebox{\textwidth}{!}{
\begin{tabular}{lcccccccc}
\toprule
& \multicolumn{4}{c}{\textbf{CIFAR-100}} 
& \multicolumn{4}{c}{\textbf{Tiny-ImageNet}} \\
\cmidrule(lr){2-5} \cmidrule(lr){6-9}
\textbf{UCL}
& \textbf{ResNet18} 
& \textbf{ResNet34} 
& \textbf{ResNet50} 
& \textbf{ViT-B/16}
& \textbf{ResNet18} 
& \textbf{ResNet34} 
& \textbf{ResNet50} 
& \textbf{ViT-B/16} \\
\midrule
InfoNCE     & 71.76 & 74.05 & 75.55 & 86.44  & 244.79 & 250.85 & 277.16 & 288.26 \\
Standard OT & 72.38 & 76.26 & 76.37 & 86.80  & 231.37 & 282.39 & 351.46 & 286.67 \\
InvaSpread  & 75.21 & 80.41 & 83.78 & 87.61  & 257.83 & 324.82 & 360.12 & 287.85 \\
OT-PushPull & 73.54 & 76.28 & 77.34 & 87.89  & 294.22 & 283.76 & 351.88 & 288.30 \\
MMIOT       & 79.28 & 82.51 & 82.30 & 100.57 & 273.87 & 294.03 & 376.97 & 300.93 \\
\bottomrule
\end{tabular}
}
\caption{Average training time (seconds per epoch) for UCL methods on CIFAR-100 and Tiny-ImageNet.}
\end{table*}

Despite its higher theoretical computational cost, MMIOT is only moderately slower in practice. Moreover, pretraining is performed only once, while the resulting representations can be reused across multiple downstream tasks. Therefore, selecting a more effective configuration during pretraining may reduce the overall computational cost of subsequent training, making this trade-off worthwhile.

\end{document}